\documentclass[sn-mathphys,Numbered]{sn-jnl}%

\usepackage{graphicx}%
\usepackage{multirow}%
\usepackage{amsmath,amssymb,amsfonts}%
\usepackage{amsthm}%
\usepackage{mathrsfs}%
\usepackage[title]{appendix}%
\usepackage{xcolor}%
\usepackage{textcomp}%
\usepackage{manyfoot}%
\usepackage{booktabs}%
\usepackage{algorithm}%
\usepackage{algorithmic}%

\usepackage{listings}%
\newtheorem{theorem}{Theorem}
\newtheorem{assumption}{Assumption}
\usepackage{hyperref}       % hyperlinks

\newtheorem{lemma}{Lemma}%

\newcommand*{\R}{{\mathbb R}}
\newcommand*{\E}{{\mathbb E}}

\newcommand{\sm}[2]{\begin{smallmatrix}\item1\\\item2 \end{smallmatrix}}

\def\R{\mathbb{R}}
\def\S{\mathcal{S}}
\newcommand\ddfrac[2]{\frac{\displaystyle \item1}{\displaystyle \item2}}

\newcommand{\mc}[1]{\mathbb{\item1}}

\begin{document}

\title[Implicitly normalized forecaster with clipping for linear and non-linear heavy-tailed multi-armed bandits.]{Implicitly normalized forecaster with clipping for linear and non-linear heavy-tailed multi-armed bandits.}

\author[1,2,4]{\fnm{Yuriy} \sur{Dorn}}\email{dornyv@yandex.ru}

\author[2]{\fnm{Nikita} \sur{Kornilov}}\email{kornilov.nm@phystech.edu}

\author[2]{\fnm{Nikolay} \sur{Kutuzov}}\email{kutuzov.nv@phystech.edu}

\author[5]{\fnm{Alexander} \sur{Nazin}}\email{nazin.alexander@gmail.com}

\author[6]{\fnm{Eduard} \sur{Gorbunov}}\email{eduard.gorbunov@mbzuai.ac.ae}

\author[2, 3, 7]{\fnm{Alexander} \sur{Gasnikov}}\email{gasnikov@yandex.ru}

\affil[1]{\orgname{MSU Institute for Artificial Intelligence}, \orgaddress{\city{Moscow}, \country{Russia}}}

\affil[2]{\orgname{Moscow Institute of Physics and Technology}, \orgaddress{\city{Dolgoprudny}, \country{Russia}}}

\affil[3]{\orgname{Skoltech}, \orgaddress{\city{Moscow}, \country{Russia}}}

\affil[4]{\orgname{Institute for Information Transmission Problems}, \orgaddress{\city{Moscow}, \country{Russia}}}

\affil[5]{\orgname{V.A. Trapeznikov Institute of Control Sciences of Russian Academy of Sciences}, \orgaddress{\city{Moscow}, \country{Russia}}}

\affil[6]{\orgname{Mohamed bin Zayed University of Artificial Intelligence}, \orgaddress{\city{Abu-Dhabi}, \country{United Arab Emirates}}}

\affil[7]{\orgname{ISP RAS Research Center for Trusted Artificial Intelligence}, \orgaddress{\city{Moscow}, \country{Russia}}}


\maketitle

\begin{abstract}

  Abstract: The Implicitly Normalized Forecaster (INF) algorithm is considered to be an optimal solution for adversarial multi-armed bandit (MAB) problems. However, most of the existing complexity results for INF rely on restrictive assumptions, such as bounded rewards. Recently, a related algorithm was proposed that works for both adversarial and stochastic heavy-tailed MAB settings. However, this algorithm fails to fully exploit the available data.

In this paper, we propose a new version of INF called the Implicitly Normalized Forecaster with clipping (INF-clip) for MAB problems with heavy-tailed reward distributions. We establish convergence results under mild assumptions on the rewards distribution and demonstrate that INF-clip is optimal for linear heavy-tailed stochastic MAB problems and works well for non-linear ones. Furthermore, we show that INF-clip outperforms the best-of-both-worlds algorithm in cases where it is difficult to distinguish between different arms.

\end{abstract}


%%%%%%%%%%%%%%% Introduction %%%%%%%%%%%%%%%%%%%%%%%%%
\section{Introduction}\label{sec:intro}

\subsection{Multi-armed bandits}

The multi-armed bandit (MAB) problem covers an essential area of optimal sequential decision making for discrete systems under uncertainty, adjoining such important areas as optimization methods, optimal control, learning and games,  behavior and automata, etc. In this model decision maker encounter a finite number of possible decision options (arms), and each applied decision leads to random reward with unknown distribution. An agent then aims to maximize the received rewards in a given time horizon. 

The first published paper by Thompson 
\cite{thompson1933likelihood} is close to the essence of the bandit problem setup. Formal problem setting was proposed by H. Robbins
\cite{robbins1952some} and gave rise to a systematic study (\cite{berry1985bandit},
\cite{gittins2011multi},
\cite{cesa2006prediction}, \cite{slivkins2019introduction}
and \cite{bubeck2012regret}) of bandit issues. MAB ideas influence many other closely related research fields, such as reinforcement learning
\cite{sutton1998reinforcement} and biology system modeling and automation theory (\cite{tsetlin1969issledovaniya},
\cite{tsetlin1973automaton}, \cite{varvsavskij1973kollektivnoe}, and \cite{nazin1986adaptive}). 

Another classical MAB problem, namely the adversarial MAB problem, assumes that losses are not stochastic, but are chosen by the adversarial agent \cite{auer2002nonstochastic}. The notion of average regret can still be applied, but the stochastic uncertainty comes from the randomization in the forecaster's strategy and not from the environment. The main approach to dealing with adversarial multi-armed bandits is based on online convex optimization (see \cite{flaxman2004online}, \cite{orabona2019modern} and \cite{hazan2016introduction} for reference).

\subsection{Algorithms for stochastic MAB}

Most algorithms for stochastic multi-armed bandits based on the UCB-strategy are proposed in the seminal paper \cite{auer2002finite}. This strategy assumes that one can construct Upper Confidence Bounds for the average reward using empirical means and confidence intervals. When considering stochastic multi-armed bandits, the vast majority of authors assume a sub-Gaussian distribution of rewards. But in some practical cases (for example, in finance \cite{Rachev2003}  or blockchain networks \cite{wang2019flash}) the rewards distribution have heavy-tails or can be adversarial. In the network analysis context, heavy tails can be found in online social networks \cite{choi2020rumor}, extreme random graphs \cite{dhara2020heavy}, the World Wide Web graph, and many other applications with power-law-distributed metrics \cite{barabasi1999emergence}. This has led to  extensive development of new algorithms for heavy-tailed stochastic MAB.

Probably the first of the UCB-based strategies for a non-(sub-)gaussian distribution, namely $\psi$-UCB strategy, is suboptimal (up to logarithmic factor), but requires that the distribution of the rewards has a finite moment generation function. In \cite{bubeck2013bandits}, the authors propose the Robust UCB algorithm and show that it is optimal under the assumption that the reward distribution for each action $i$ satisfies \begin{equation}\label{ass1}
    \E_{X \sim \mathbf{p}(a_i)} |X|^{1+\alpha} \leq M^{1+\alpha}
\end{equation} for some $\alpha \in (0, 1]$ and $M>0$. 

In \cite{medina2016no} authors consider linear stochastic multi-armed bandits (LinBET) and propose new algorithms, one based on dynamic truncation with regret bound $O \left ( T^{\frac{2+\alpha}{2(1+\alpha)} } \right )$ and another based on median of means with regret bound $O \left ( T^{\frac{1+2\alpha}{1+3\alpha}}  \right )$. Later in \cite{shao2018almost} authors prove $\Omega \left (T^{\frac{1}{1+\alpha}} \right )$ lower bound for LinBET problem setting and propose two UCB-based algorithms, which continue ideas from \cite{medina2016no} and enjoy sub-optimal (up to logarithmic factor) regret bound with high probability. There are few other works that exploit idea of mixing robust statistical estimators with UCB (\cite{zhong2021breaking}) which results in sub-optimal algorithms. In \cite{lu2019optimal} authors propose near optimal algorithm for Lipschitz stochastic MAB with heavy tails with regret bound $O \left ( T^{\frac{1+d_z\alpha}{1+\alpha+d_z\alpha}}  \right )$, where $d_z$ is zooming dimension. Another suboptimal algorithm for stochastic MAB was proposed in \cite{lee2020optimal}.

\subsection{Algorithms, based on online convex optimization approach}

Recently there was a new approach, where algorithms based on online mirror descent have been proposed for heavy-tailed stochastic MAB (\cite{zimmert2019optimal}) just like Exp3 was proposed for adversarial bandits(\cite{auer2002nonstochastic}). This leads us to another relevant line of works, related to stochastic optimization with heavy-tailed noise in gradients, i.e. under assumption $\E\left [ \|g_t\|^{1+\alpha}_{\infty} \right ] \leq M^{1+\alpha}$.
When $\alpha=1$ the above assumption reduces to the standard bounded second moment assumption \cite{nemirovski2009robust}. This case is studied relatively well both for standard minimization and online learning. As for the setup with $\alpha \in (0,1)$, the study of the methods' convergence is still under development. For minimization problems, the first in-expectation results are derived in \cite{nemirovskij1983problem} for convex problems. The work of \cite{zhang2020adaptive} extends the result to the strongly convex case and for the non-convex smooth case derives a new result under a similar assumption:
\begin{equation}
    \E\left [ \|g_t - \nabla f(x_t)\|^{1+\alpha}_{2} \right ] \leq \sigma^{1+\alpha}. \label{eq:bounded_central_alpha_moment}
\end{equation}
In \cite{vural2022mirror}, the results are extended to the case of uniformly convex functions. High-probability convergence results under condition \eqref{eq:bounded_alpha_moment} are derived in \cite{cutkosky2021high} in the non-convex case and under condition \eqref{eq:bounded_central_alpha_moment} in \cite{sadiev2023high}, as well as in the convex, strongly convex, non-convex cases in \cite{sadiev2023high}. Optimal (up to logarithmic factors) high-probability regret bounds under assumption $$\E_{X \sim \mathbf{p}(a_i)} |X|^{1+\alpha} \leq u$$ are obtained in \cite{zhang2022parameter}.

In the context of heavy-tailed MAB, the main example is \cite{huang2022adaptive}, where the authors propose the HTINF and AdaTINF algorithms, which perform almost optimally for both stochastic and adversarial settings and which exploit a solution scheme for adversarial bandits. Lately \cite{dann2023blackbox} shows that this algorithm is also almost optimal for linear MAB. 

In \cite{huang2022adaptive}, the authors also use an algorithmic scheme for adversarial bandits, i.e., it reconstruct losses for all arms (from bandit feedback to estimated full feedback) and use online mirror descent. As always, the heavy-tailed distribution affects the loss estimate. To solve this problem, the authors proposed to use a skip threshold, i.e. to truncate the samples if the loss is too big. This solves the problem but increases the number of samples since not all data are used.

Most of the algorithms mentioned before consider stochastic linear MAB. Even when using the algorithms for adversarial MAB for heavy-tailed stochastic MAB, one usually assumes the linearity condition. On the other hand, the so-called nonlinear bandits \cite{bubeck2012regret}, where the loss function (derived from online optimization notation) is not linear, are left without care.

\subsection{Our contributions and related works} \label{S:contrib}

Our work is closely related to \cite{huang2022adaptive}. In this paper, we propose an Implicitly Normalized Forecaster with clipping (INF-clip) algorithm, which is likewise based on online mirror descent for stochastic MAB with a heavy-tail reward distribution, but instead of truncation we use clipping. This allows us to use samples that were truncated in \cite{huang2022adaptive}, which leads to a better use of data and better convergence properties in experiments.

We show that the INF-clip guaranties a minimax optimal regret bound $O \left ( M n^{\frac{\alpha}{1+\alpha}} T^{\frac{1}{1+\alpha}}  \right )$ for heavy-tailed MAB with $n$ arms, i.e. if assumption \ref{ass1} holds for each arm. 

We also show that our algorithm can be applied to nonlinear bandits settings with adversarial noise setup and provide the corresponding convergence analysis. To the best of our knowledge, this is the first result with a convergence guarantee for nonlinear bandits with  $\alpha \in (0, 1]$.

\begin{table}[h!]
\centering
\begin{tabular}{|l||l||l||c||}
 \hline
 \small{Algorithm} & \small{Stochastic} & \small{Adversarial} & \small{Regret for heavy-tailed} \\
 & & & \small{stochastic MAB}\\
 \hline\hline
  Robust UCB (\cite{bubeck2013bandits}, 2013)&+&-&$O \left ( T^{\frac{1}{1+\alpha}}\cdot M \cdot n^{\frac{\alpha}{1+\alpha}} \cdot (\log T)^{\frac{\alpha}{1+\alpha}} \right)$\\
 APE (\cite{lee2020optimal}, 2020)&+&-&$O \left ( T^{\frac{1}{1+\alpha}}\cdot M \cdot n^{\frac{\alpha}{1+\alpha}} \cdot \log K \right)$\\
   \small{$\frac{1}{2}$-Tsallis-INF} (\cite{zimmert2019optimal}, 2019) & + & + &  $O \left ( \sqrt{Tn} \right)$ (only for $\alpha = M = 1$)\\
  \small{HTINF (\cite{huang2022adaptive}, 2022)} & + & +&  $O \left ( T^{\frac{1}{1+\alpha}}\cdot M \cdot n^{\frac{\alpha}{1+\alpha}} \right)$\\
 \small{INF-clip (this work)} & + & + &  $O \left ( T^{\frac{1}{1+\alpha}}\cdot M \cdot n^{\frac{\alpha}{1+\alpha}} \right)$\\
  \hline
\end{tabular}
 \caption{An overview of related algorithms. Here $T$ is the number of rounds, $n$ is the number of arms}
\end{table}

\subsection{Paper organization}

In Section \ref{sec:Prem}, we introduce a formal problem statement for stochastic MAB with heavy tails. We also introduce a general scheme on how to use online mirror descent to solve MAB problems. Next, in Sections \ref{sec:INF-clip} and \ref{sec:nINF-clip} we introduce Implicitly Normalized Forecaster with Clipping algorithm for stochastic heavy-tailed MAB, linear and nonlinear cases, respectively. Also we introduce convergence analysis for each algorithm. We put the formal statement for the nonlinear bandits in Section  \ref{sec:nINF-clip} to make it easier to read. In Section  \ref{sec:Exp} we show the results of computational experiments. Some of the missing proofs from Section \ref{sec:nINF-clip} go to Section 
 \ref{sec: proofs}.

\section{Preliminaries}\label{sec:Prem}

\subsection{Notations}
 For $p \in [1,2]$ notation $||\cdot||_p$ is used for the standard $l_p$-norm, i.e.  $||x||_p = \left(\sum_{i=1}^n |x_i|^p\right)^{1/p}$. The corresponding dual norm is $||y||_{q}$ with $\frac{1}{p} + \frac{1}{q}=1$. We interchangeably use the notations $\|x\|_2$ and $\|x\|$.
 
 Set $B^p = \{x \in \R^n\mid ||x||_p \leq 1\}$ is a $p$-ball with center at $0$ and radius $1$, and $S^p = \{x \in \R^n\mid ||x||_p = 1\}$ is a $p$-sphere with center at $0$ and radius $1$. Finally, for $\tau > 0$ and convex set $\S\subset \R^n$ we denote $\S_\tau = \S + \tau\cdot B^2$. Set $\triangle_n = \left \{x \in R^n| \quad \sum_{i=1}^n x_i=1, \quad x \geq 0  \right \}$.

 The expectation of a random variable $X$ is denoted by $\E [X]$. The expected value for a given function $\phi(\cdot)$ that takes a random variable $\xi$ with a probability distribution $\mathcal{D}$ as argument is denoted by $\E_{\xi\sim \mathcal{D}} [\phi(\xi)]$. To simplify the notation, we interchangeably use the notations $\E_{\xi\sim \mathcal{D}} [\phi(\xi)]$, $\E_{\xi} [\phi(\xi)]$ and $\E_{\mathcal{D}} [\phi(\xi)]$.
 
 In case when $X$ represents the reward obtained by the algorithm in the stochastic MAB setting with $n$ arms, there are two sources of stochasticity: the stochastic nature of rewards for each arm and the randomization in the algorithm's arm-picking procedure. Suppose that the algorithm chooses $i$-th arm with probability $x_i$, and the reward for the $i$-th arm is represented by the random variable $X_i$. Then the notation $\E_{x} [X]$ represents the full expectation for $X$ with respect to the given probability distribution $x \in \triangle_n$, where $x = (x_1, \dots, x_n)$, i.e. $$\E_x[X] = \sum_{1\leq i\leq n} x_i \cdot \E[X_i].$$

\subsection{Stochastic MAB problem}

The stochastic MAB problem can be formulated as follows: an agent at each time step $t = 1,\dots, T$ chooses action $A_t$ from given action set $\mathcal{A} = (a_1, \dots, a_n)$ and suffer stochastic loss (or get reward) $l_t(A_t)$. Agent can observe losses only for one action at each step, namely, the one that he chooses. For each action $a_i$ there exists a probability density function for losses $\mathbf{p}(a_i)$ with mean loss $\nu_i$, an agent doesn't know PDF of any action in advance. At each round $t$ when action $a_i$ is chosen (i.e. $A_t = a_i$) stochastic loss $l_t(A_t)$ sampled from $\mathbf{p}(a_i)$ independently.

Agent goal is to minimize \textit{average regret}:
\begin{equation*}
    \frac{1}{T} \E [\mathcal{R}_T (u)] = \frac{1}{T} \E \left [ \sum_{t=1}^T \left (l_t(A_t) - l_t(u) \right ) \right ]
\end{equation*}
for a fixed competitor $u \in \mathcal{A}$.

\textbf{\textit{Assumption (heavy tails) 1:}}
There exist $\alpha >0$ and $M>0$, such that for each action $a_i$ (arm) stochastic reward $X_i$ satisfy
\begin{equation}
    \E\left [ |X_i|^{1+\alpha} \right ] \leq M^{1+\alpha}, \label{eq:bounded_alpha_moment}
\end{equation}
where  $\alpha \in (0,1]$. 

\subsection{Algorithmic scheme for MAB problems based on online mirror descent}

Online convex optimization algorithms (\cite{hazan2016introduction}, \cite{orabona2019modern}) are the standard choice for solving adversarial MAB problems, but have recently been considered for stochastic MAB with heavy-tailed noise, adversarial distributed noise and for non-linear MAB. In the following, we identify each action $a_i$ (arm) with index $i$.

The basic procedure can be described as follows:

\begin{enumerate}
    \item At each step $t$, agent choose action $A_t$ randomly from distribution $x_t \in \triangle_n$, i.e. $P[A_t = i] = x_{t, i}$ for each $i=1, \dots, n$.
    \item Play $A_t$ and observe reward for the action.
    \item Construct (artificial) loss function $l_t(x_t) = \langle g_t, x_t \rangle$ where $g_{t, i}$ is unbiased (w.r.t. $x_t$) loss estimator for action $i$. 
    \item Use online mirror descent to solve online convex optimization problem for artificial losses.
\end{enumerate}
This approach is closely related to prediction with expert advice (see Hedge and other related algorithms \cite{littlestone1994weighted}, \cite{cesa1997use}, \cite{freund1997decision}), namely, one can say that step (3) of the procedure serves to estimate the full feedback from the bandit's feedback, after that use an aggregated predictor, constructed via prox-function from online mirror descent. This scheme was introduced in the Exp3 algorithm for adversarial bandits in the seminal paper \cite{auer2002nonstochastic}.

\section{Implicitly Normalized Forecaster with clipping for linear heavy-tailed multi-armed bandits} \label{sec:INF-clip}

In this section, we present the Implicitly Normalized Forecaster with clipping (INF-clip). INF is an online mirror descent algorithm with Tsallis entropy as the prox applied to the multi-armed bandit (MAB) problem. 

Since the (sub)gradient is not observed in the MAB setting, it is substituted by a stochastic approximation. A standard choice of stochastic approximation of (sub)gradient $g_t \in \partial l(x_t)$ is an importance-weighted estimator: $$\bar{g}_{t,i} = \begin{cases} \frac{g_{t,i}}{x_{t, i}} & \text{if } i = A_t\\ 0 & \text{otherwise} \end{cases}$$

This is an unbiased estimator, i.e. $\mathrm{E}_{x_t}[\bar{g}_t] = g_t$. One of the drawbacks of this estimator is that in the case of small $x_{t,i}$ the value of $\bar{g}_{t,i}$ can be arbitrarily large. In the case when the distribution of $g_t$ has heavy tails (i.e. $\|g_t\|_{\infty}$ can be large with high probability), this drawback can be amplified. Thus, we propose to use the clipped losses $clip(g_{t,i}, \lambda) = \min \{g_{t,i}, \lambda \}$ instead of $g_{t,i}$ to compute the biased estimator $\hat{g}_{t,i}$. This helps to handle the bursts of losses.

\begin{algorithm}[H]
\caption{Implicitly Normalized Forecaster with Clipping}
\label{alg:base_alg}
  \begin{algorithmic}[1]
  
  \REQUIRE Starting point $x_0 =  (1/n, \dots, 1/n  )$,  parameters $\mu>0$, $q \in (0, 1)$.
  \FOR{$t=0, \ldots, T-1$}
  \STATE Draw $A_t$ with $P(A_t = i) = x_{t, i}$, $i=1,\dots, n$,
  \STATE Play $A_t$ and observe reward $g_{t, A_t}$,
  \STATE Construct estimate $\hat{g}_{t,i} = \begin{cases} \frac{clip(g_{t,i}, \lambda)}{x_{t, i}} & \text{if } i = A_t \\ 0 & \text{otherwise} \end{cases} $, \quad $i=1,\dots, n$,
  \STATE Compute
\begin{equation*}
    x_{t+1} = \arg \min_{x \in \triangle_n} \left [ \mu x^{\mathtt{T}} \hat{g}_t - \frac{1}{1-q
    } \sum_{i=1}^n x_i^q + \frac{q}{1-q}\sum_{i=1}^n x_{t, i}^{q-1}x_i    \right ].
\end{equation*}
  \ENDFOR 
  \STATE \textbf{return} $x_T$.
\end{algorithmic}
\end{algorithm}

\bigskip

\begin{lemma}
\label{regretest1}
Suppose that Algorithm \ref{alg:base_alg}  generates the sequences $\{x_t\}_{t=1}^T$ and $\{\hat{g}_t\}_{t=1}^T$, then for any $u \in \triangle_n$ holds:

\begin{equation}
     \sum_{t=1}^T \langle \hat{g}_t, x_t - u \rangle   \leq \frac{n^{1-q} - \sum_{i=1}^n u_i^q}{\mu(1-q)} + \frac{\mu}{2q}\sum_{t=1}^T \sum_{i=1}^n  \hat{g}_{t, i}^2x_{t,i}^{2-q}.
\end{equation}
\end{lemma}

\textit{Proof:}

Consider Tsallis entropy $$\psi_q(x) = \begin{cases} \frac{1}{1-q} \left ( 1 - \sum_{i=1}^n x_i^q \right ), \quad q \in (0, 1)\\
- \sum_{i=1}^n x_i \ln x_i, \quad q=1 \end{cases}.$$ 

Consider \textit{Bregman divergence} $B_{\psi_q} (x, y)$:
$$
B_{\psi_q} (x, y) = \psi_q(x) - \psi_q(y) -\langle\nabla\psi_q(y), x-y\rangle 
$$
$$
=\frac{1}{1-q} \left ( 1 - \sum_{i=1}^n x_i^q \right ) - \frac{1}{1-q} \left ( 1 - \sum_{i=1}^n y_i^q \right ) - \sum_{i=1}^n \frac{-qy_i^{q-1}}{1-q}(x_i-y_i)
$$
$$
=\frac{-1}{1-q}\sum_{i=1}^n x_i^q+\frac{1}{1-q}\sum_{i=1}^n y_i^q + \sum_{i=1}^n \frac{qy_i^{q-1}}{1-q}(x_i-y_i).
$$
Note that the last operation of the algorithm can be considered as an online mirror descent (OMD) step with the Tsallis entropy $\psi_q(x) = \frac{1}{1-q} \left ( 1 - \sum_{i=1}^n x_i^q \right )$ ($q \in (0, 1]$) used as prox: 
\begin{align*}
    &x_{t+1} = \arg \min_{x \in \triangle_n} \left [ \mu x^{\mathtt{T}} \hat{g}_t - \frac{1}{1-q
    } \sum_{i=1}^n x_i^q + \frac{q}{1-q}\sum_{i=1}^n x_{t, i}^{q-1}x_i    \right ] \\
    &=\arg \min_{x \in \triangle_n} \left [ \mu x^{\mathtt{T}} \hat{g}_t - \frac{1}{1-q
    } \sum_{i=1}^n x_i^q + \frac{q}{1-q}\sum_{i=1}^n x_{t, i}^{q-1}(x_i-x_{t,i}) + \frac{1}{1-q}\sum_{i=1}^n x_{t,i}^q   \right ]\\
    &=\arg \min_{x \in \triangle_n} \left [ \mu x^{\mathtt{T}} \hat{g}_t + B_{\psi_q}(x, x_t)   \right ].
\end{align*}

Thus standard inequation for OMD holds:
\begin{equation}\label{OMD1}
    \langle \hat{g}_t, x_t - u \rangle \leq  \frac{1}{\mu} \left [ B_{\psi_q}(u, x_t) - B_{\psi_q}(u, x_{t+1}) - B_{\psi_q}(x_{t+1}, x_t) \right]  +  \langle \hat{g}_t, x_t - x_{t+1} \rangle.
\end{equation}
From Tailor theorem we have 
\begin{equation*}
    B_{\psi_q}(z, x_t) = \frac{1}{2} (z - x_t)^T \nabla^2 \psi_q(y_t) (z -  x_t) = \frac{1}{2}\|z - x_t\|^2_{\nabla^2\psi_q(y_t)}
\end{equation*}
for some point $y_t \in [z, x_t]$.

Hence we have
\begin{align*}
    &\langle \hat{g}_t, x_t - x_{t+1} \rangle - \frac{1}{\mu}B_{\psi_q}(x_{t+1}, x_t) \leq \max_{z \in R^n_+} \left [\langle \hat{g}_t, x_t - z \rangle - \frac{1}{\mu}B_{\psi_q}(z, x_t) \right ]\\
    &=\left [\langle \hat{g}_t, x_t - z^*_t \rangle - \frac{1}{\mu}B_{\psi_q}(z^*_t, x_t) \right ] \\
    &\leq \frac{\mu}{2}\|\hat{g}_t\|^2_{(\nabla^2 \psi_q (y_t))^{-1}} + \frac{1}{2}\|z^* - x_t\|^2_{\nabla^2\psi_q(y_t)} - \frac{1}{\mu}B_{\psi_q}(z^*, x_t) \\
    &=\frac{\mu}{2}\|\hat{g}_t\|^2_{(\nabla^2 \psi_q (y_t))^{-1}},
\end{align*}
where $z^* = \arg \max_{z \in \R^n_+} \left [\langle \hat{g}_t, x_t - z \rangle - \frac{1}{\mu}B_{\psi_q}(z, x_t) \right ]$.

Proceeding with (\ref{OMD1}), we get:
\begin{equation*}
    \langle \hat{g}_t, x_t - u \rangle \leq \frac{1}{\mu} \left [ B_{\psi_q}(u, x_t) - B_{\psi_q}(u, x_{t+1})  \right] + \frac{\mu}{2}\|\hat{g}_t\|^2_{(\nabla^2 \psi_q (y_t))^{-1}}.
\end{equation*}
Sum over $t$ gives
$$
\sum_{t=1}^T \langle \hat{g}_t, x_t - u \rangle    \leq \frac{B_{\psi_q} (x_1, u)}{\mu} + \frac{\mu}{2}\sum_{t=1}^T \hat{g}_t^T \left ( \nabla^2 \psi_q(y_t) \right )^{-1}\hat{g}_t
$$

\begin{equation}\label{sbound1}
    =\frac{n^{1-q} - \sum_{i=1}^n u_i^q}{\mu(1-q)} + \frac{\mu}{2q}\sum_{t=1}^T \sum_{i=1}^n  \hat{g}_{t, i}^2y_{t,i}^{2-q},
\end{equation}
where $y_t \in [x_t, z_t^*]$, $z_t^* = \arg \max_{z \in R^n_+} \left [\langle \hat{g}_t, x_t - z \rangle - \frac{1}{\mu}B_{\psi_q}(z, x_t) \right ]$.

From the first-order optimality condition for $z^*_t$ we obtain
$$
-\frac{\mu (1-q)}{q} \hat{g}_{t, i} + (x_{t, i})^q = (z_{t, i}^*)^q
$$
and since $\hat{g}_{t, i} \geq 0$ we get $z_{t, i}^* \leq x_{t, i}$. 

Thus (\ref{sbound1}) becomes 
\begin{equation*}
    \sum_{t=1}^T \langle \hat{g}_t, x_t - u \rangle   \leq \frac{n^{1-q} - \sum_{i=1}^n u_i^q}{\mu(1-q)} + \frac{\mu}{2q}\sum_{t=1}^T \sum_{i=1}^n  \hat{g}_{t, i}^2x_{t,i}^{2-q}
\end{equation*}
and concludes the proof.

\begin{lemma}
\label{sad}
    \textbf{[Lemma 5.1 from \cite{sadiev2023high}]} Let $X$ be a random vector in $\R^n$ and $\bar{X} = clip(X, \lambda) = X \cdot \min \left \{ 1, \frac{\lambda}{\|X\|} \right \}$, then 
\begin{equation}
     \|\bar{X} - \E[\bar{X}]\| \leq 2\lambda.
\end{equation}
Moreover, if for some $\sigma \geq 0$ and $\alpha \in [1, 2)$
\begin{equation*}
    \E[X] = x \in \R^n, \quad \E[\|X - x\|^{\alpha}] \leq \sigma^{\alpha}
\end{equation*}
and $\|x\| \leq \frac{\lambda}{2}$, then 
\begin{align}
    \left \| \E[\bar{X}] - x \right \| &\leq \frac{2^{\alpha} \sigma^{\alpha}}{\lambda^{\alpha-1}},\\
    \E\left [ \left \| \bar{X} - x   \right \|^2 \right] &\leq 18 \lambda^{2-\alpha}\sigma^{\alpha},\\
    \E\left [ \left \| \bar{X} - \E[\bar{X}]   \right \|^2 \right] &\leq 18 \lambda^{2-\alpha}\sigma^{\alpha}.
\end{align}
\end{lemma}

\begin{lemma}
    \label{regretest2}
Suppose that Algorithm \ref{alg:base_alg} generates the sequences $\{x_t\}_{t=1}^T$ and $\{\hat{g}_t\}_{t=1}^T$, $\E\left [ g_{t, i}^{1+\alpha} \right ] \leq M^{1+\alpha}$, $\E [g_{t, i}] \leq \frac{\lambda}{2}$, then for any $u \in \triangle_n$ holds:
\begin{equation}
    \E_{x_t} \left [\langle \bar{g}_t - \hat{g}_t, x_t - u \rangle \right ] \leq \frac{2^{\alpha+1} M^{\alpha+1}}{\lambda^{\alpha}}.
\end{equation}
\end{lemma}

\textit{Proof:}

Consider a non-negative random variable $X \in \R_+$ 
such that $\E[X^{1+\alpha}] \leq M^{1+\alpha}$ and $\E[X] = x$. From Jensen's inequality:
\begin{equation*}
    x = \E[X] = \left ( \left (\E[X] \right )^{1+\alpha} \right )^\frac{1}{1+\alpha} \leq \left (\E[X^{1+\alpha}] \right )^{\frac{1}{1+\alpha}} \leq \left (M^{1+\alpha} \right )^{\frac{1}{1+\alpha}} = M.
\end{equation*}
Hence
\begin{align}
    &\E[|X - x|^{1+\alpha}] \leq \E[X^{1+\alpha}] + x^{1+\alpha} \leq 2M^{1+\alpha}.
\end{align}
Thus from Lemma \ref{sad} we get:
\begin{equation*}
    \left \| \E[\bar{X}] - x \right \| \leq \frac{2^{\alpha+1} M^{\alpha+1}}{\lambda^{\alpha}}.
\end{equation*}
Finally taking $g_{t,i}$ as $X$ and applying the previous results we get
\begin{align*}
    &\E_{x_t} \left [\langle \bar{g}_t - \hat{g}_t, x_t - u \rangle \right ] =  \langle \E_{x_t} \left [\bar{g}_t - \hat{g}_t \right ], x_t - u \rangle \\
    &=\sum_{i=1}^n (x_{t,i} - u_i) \E_{x_t} \left [\bar{g}_{t, i} - \hat{g}_{t, i} \right ] =\sum_{i=1}^n (x_{t,i} - u_i)\cdot x_{t,i} \cdot \E \left [\bar{g}_{t, i} - \hat{g}_{t, i} \right ]\\
    &\leq \max_{x \in \triangle_n} \left [\sum_{i=1}^n (x_{i} - u_i)\cdot x_{i} \cdot \E \left [\bar{g}_{t, i} - \hat{g}_{t, i} \right ] \right ]\\
    &\leq \max_{1 \leq i \leq n} |\E[g_{t, i}] - \E[clip(g_{t, i}, \lambda)]|\leq \frac{2^{\alpha+1} M^{\alpha+1}}{\lambda^{\alpha}}.
\end{align*}

\begin{lemma}
    \label{regretest3}
    Suppose that Algorithm \ref{alg:base_alg} generates the sequences $\{x_t\}_{t=1}^T$ and $\{\hat{g}_t\}_{t=1}^T$ and assumption 1 is satisfied, then:
\begin{equation}
\E\left[\sum_{i=1}^{n} \hat{g}_{t, i}^2 \cdot x_{t, i}^{\frac{3}{2}}\right] \leq \sqrt{n} \lambda^{1-\alpha}\cdot M^{1+\alpha}.
\end{equation}
\end{lemma}

\textit{Proof:}

Estimation of the clipped losses variance:
\begin{equation}\label{est1}
\E\left[\hat{g}_{t, i}^2\right] = \E\left[\hat{g}_{t, i}^{1-\alpha} \cdot \hat{g}_{t, i}^{1+\alpha} \right] \leq \lambda^{1-\alpha} \cdot \E\left[\hat{g}_{t, i}^{1+\alpha}\right] \leq \lambda^{1-\alpha} \cdot M^{1+\alpha}.
\end{equation}

Then
\begin{equation*}
\E\left[\sum_{i=1}^{n} \hat{g}_{t, i}^2 \cdot x_{t, i}^{\frac{3}{2}}\right] \leq \E\left[\sum_{i=1}^{n} \hat{g}_{t, i}^2 \cdot \sqrt{x_{t, i}}\right] \leq \E\left[\sqrt{\sum_{i=1}^{n}\hat{g}_{t, i}^2} \cdot \sqrt{\sum_{i=1}^{n}\hat{g}_{t, i}^2 \cdot x_{t, i}}\right] 
\end{equation*}
\begin{equation*}
\leq \sqrt{\E\left[\sum_{i=1}^{n}\hat{g}_{t, i}^2\right]} \cdot \sqrt{\E\left[\sum_{i=1}^{n}\hat{g}_{t, i}^2 \cdot x_{t, i}\right]}.
\end{equation*}

From (\ref{est1}):
\begin{equation}\label{est2}
    \E\left[\sum_{i=1}^{n}\hat{g}_{t, i}^2\right] \leq n \cdot \lambda^{1-\alpha} \cdot M^{1+\alpha},
\end{equation}
\begin{equation}\label{est3}
    \E\left[\sum_{i=1}^{n}\hat{g}_{t, i}^2 \cdot x_{t, i}\right] = \sum_{i=1}^{n}x_{t, i} \cdot \E\left[\hat{g}_{i, t}^2\right] \leq \max_{1 \leq i \leq n} \E\left[\hat{g}_{i, t}^2\right] \leq \lambda^{1-\alpha} \cdot M^{1+\alpha}.
\end{equation}

Summarizing (\ref{est2}) and (\ref{est3}):
\begin{equation*}
\sqrt{\E\left[\sum_{i=1}^{n}\hat{g}_{t, i}^2\right]} \cdot \sqrt{\E\left[\sum_{i=1}^{n}\hat{g}_{t, i}^2 \cdot x_{t, i}\right]} \leq \sqrt{n} \lambda^{1-\alpha} \cdot M^{1+\alpha}.
\end{equation*}

Now we are ready to formulate our main result:

\begin{theorem}
    \label{avrregretest}
    Let assumption 1 be satisfied and the clipping parameter $\lambda = T^{\frac{1}{\left( 1 + \alpha \right)}} \cdot \frac{\left(\frac{2 \alpha}{1 - \alpha} \right) ^ {\frac{2}{1 + \alpha}}}{\left( 8n \right) ^ {\frac{1}{1 + \alpha}}} \cdot M$. Then the Implicitly Normalized Forecaster with clipping at $\mu = \frac{\sqrt{2}}{\sqrt{T\lambda^{1-\alpha} M^{1+\alpha}}}$ for any fixed $u \in \triangle_n$ is satisfied:
\begin{equation}
    \frac{1}{T}\E \left [\mathcal{R}_T (u) \right ] \leq  T^{-\frac{\alpha}{1+\alpha}}\cdot M \cdot n^{\frac{\alpha}{1+\alpha}}\cdot 2^{2-\frac{\alpha^2}{1+\alpha}}\cdot \left ( \frac{\alpha}{1-\alpha}\right)^{\frac{2}{1+\alpha}}.
\end{equation}
\end{theorem}

\textbf{Proof:}
\begin{align*}
    &\E\left[\mathcal{R}_T(u)\right] = \E\left [\sum_{t=1}^T l(x_t) - \sum_{t=1}^T l(u) \right ] \leq \E\left [ \sum_{t=1}^T \langle \nabla l(x_t), x_t - u \rangle  \right]\\
    &\leq \E\left [ \sum_{t=1}^T \langle \nabla l(x_t) - \bar{g}_t, x_t - u \rangle  \right] + \E\left [ \sum_{t=1}^T \langle \bar{g}_t - \hat{g}_t, x_t - u \rangle  \right] +\E\left [ \sum_{t=1}^T \langle \hat{g}_t, x_t - u \rangle  \right]\\
    &= \E\left [ \sum_{t=1}^T \langle \bar{g}_t - \hat{g}_t, x_t - u \rangle  \right] +\E\left [ \sum_{t=1}^T \langle \hat{g}_t, x_t - u \rangle  \right] \\
    &\overbrace{ \leq}^{\text{Lemma \ref{regretest2}}} \frac{T2^{\alpha+1} M^{\alpha+1}}{\lambda^\alpha} + \E\left [ \sum_{t=1}^T \langle \hat{g}_t, x_t - u \rangle  \right]\\
    &\overbrace{ \leq}^{\text{Lemma \ref{regretest1}}} \frac{T2^{\alpha+1} M^{\alpha+1}}{\lambda^\alpha} + \E\left [ 2\frac{\sqrt{n} - \sum_{i=1}^n \sqrt{u_i}}{\mu} + \mu\sum_{t=1}^T \sum_{i=1}^n  \hat{g}_{t, i}^2x_{t,i}^{3/2}  \right]\\
    &\overbrace{ \leq}^{\text{Lemma \ref{regretest3}}} \frac{T2^{\alpha+1} M^{\alpha+1}}{\lambda^\alpha} + 2\frac{\sqrt{n}}{\mu}+ \mu T \sqrt{n} \lambda^{1-\alpha} M^{1+\alpha}\\
    &= \frac{T2^{\alpha+1} M^{\alpha+1}}{\lambda^\alpha} + 2\sqrt{2nT\lambda^{1-\alpha} M^{1+\alpha}}\\
%    &=\epsilon T + 2\sqrt{2nT\epsilon^{1 - \frac{1}{\alpha}} M^{1+\alpha}}
 %   &= T ^ \cdot M ^ {-\alpha} \cdot \left(2n \right) ^ {\frac{a}{1 + \alpha}} \cdot \left(1 + \alpha \cdot \left(-1 + 2 M ^ {1 + \alpha} \right)\right) \cdot \alpha ^ {-\frac{2 \alpha}{1 + \alpha}} \cdot \left(1 - \alpha \right) ^ {-\frac{1 - \alpha}{1 + \alpha}}.
 &=2^{2-\frac{\alpha^2}{1+\alpha}}\cdot T^{1-\frac{\alpha}{1+\alpha}}\cdot M \cdot n^{\frac{\alpha}{1+\alpha}}\cdot \left ( \frac{\alpha}{1-\alpha}\right)^{\frac{2}{1+\alpha}}.
\end{align*}

\section{Implicitly Normalized Forecaster with clipping for nonlinear
heavy-tailed multi-armed bandits}\label{sec:nINF-clip}
\subsection{Problem statement}

Consider a nonlinear MAB problem on a compact convex set $\S \subset \R^n$. 
One needs to find the sequence $\{x_t\} \subset \S$ to minimize the pseudo-regret
$$\mathcal{R}_T(\{l_t(\cdot)\}, \{x_t\}) =  \sum\limits_{t=1}^T l_t(x_t) - \min\limits_{ x\in \S} \sum\limits_{t=1}^T l_t(x).$$
After each choice of $x_t$ we obtain the loss  $\phi_t(x_t,\xi_t) = l_t(x_t, \xi_t) +\delta_t(x_t)$, where 
\begin{enumerate}
\item $l_t(x, \xi_t)$ is stochastic realization of the loss function $l_t(x)$ s. t. 
$\E_{\xi_t} [l_t(x, \xi_t)] = l_t(x)$,

    \item $\delta_t(x)$ is some adversarial noise on $t$-th step. Noise $\delta_t(x)$ is independent from $\xi_t$.  
\end{enumerate}
The choice of ${x_t}$ can be based only on available information
$$\{\phi_1(x_1,\xi_1), \dots, \phi_{t-1}(x_{t-1},\xi_{t-1}) \}.$$
Meanwhile, on each step $l_t(\cdot)$ and $\delta_t(\cdot)$ can be chosen from classes of functions which are hostile to the method used to generate $x_t$. We consider several possible scenarios for these classes.

\begin{assumption}[Convexity] \label{as: convex}
$\exists \tau > 0$, s.t. for all $t$, the functions $l_t(x, \xi)$ is convex for any $\xi$ on $\S_\tau$.
\end{assumption}
This assumption implies that $l_t(x)$ is  convex as well on $\S$.
\begin{assumption}[Bounded] \label{as: bounded}
$\exists \alpha \in (0,1], B > 0$, s.t. for all $t$ 
$$\E_{\xi}[|l_t(x, \xi)|^{\alpha+1}] \leq B^{\alpha+1} < \infty.$$
\end{assumption}

The functions $l_t(x)$ also must satisfy one of Assumptions \ref{as: Lipshcitz and bounded} or \ref{as: Smooth and bounded} below.
\begin{assumption}[Lipschitz] \label{as: Lipshcitz and bounded}
$\exists \tau > 0$, s.t. for all $t$, the functions $l_t(x, \xi)$ are $M(\xi)$ Lipschitz continuous w.r.t. $l_2$ norm, i.e., for all $x,y \in \S_\tau$
$$|l_t(x, \xi) - l_t(y, \xi)| \leq M(\xi) ||x - y||_2.$$
Moreover, $\exists \alpha \in (0,1] $  such that $\E_\xi [M^{\alpha+1}(\xi)] \leq M^{\alpha+1}< \infty.$
\end{assumption}

\begin{assumption}[Smooth] \label{as: Smooth and bounded}
$\exists \tau > 0$, s.t. for all $t$, the functions $l_t(x, \xi)$ are $L(\xi)$ smooth w.r.t. $l_2$ norm, i.e., for all $x, y \in \S_\tau$
$$|\nabla l_t(x, \xi) - \nabla l_t(y, \xi)| \leq L(\xi) ||x - y||_2.$$
Moreover, $\exists \alpha \in (0,1] $  such that $\E_\xi [L(\xi)^{\alpha+1}] \leq L^{\alpha+1} < \infty.$
\end{assumption}

Finally, we make an assumption about the adversarial noise class.
\begin{assumption}[Bounded adversarial noise]\label{as: noize}
For all $x \in \S$, the bound holds $|\delta(x)| \leq \Delta < \infty.$
\end{assumption}

\

\textbf{Gradient-free setup}

In this paper, we consider only uniform sampling from the unit Euclidean sphere, i.e. $\mathbf{e} \sim Uniform(\{\mathbf{e}: \|\mathbf{e}\|_2 = 1 \})  = U(S^2)$. 

First of all, for all $t$ we define the smoothed version of $l_t(x)$ function
\begin{equation}\label{hat_f}
    \hat{l}^\tau_t (x) = \E_{\mathbf{e} \sim U(S^2)} [l_t(x + \tau\mathbf{e})],
\end{equation}
that approximates $l_t$. Further $U(S^2)$ in $\E_{\mathbf{e} \sim U(S^2)}$ is omitted.
Following \cite{shamir2017optimal}, $\hat{l}_t^\tau(x)$ is differentiable, and the gradient  can be estimated using noisy one-point feedback on the following vector:
\begin{equation}\label{g}
g_t(x, \xi, \mathbf{e}) = \frac{n}{\tau}(\phi_t(x + \tau \mathbf{e}, \xi)) \mathbf{e} = \frac{n}{\tau}(l_t(x, \xi) +\delta_t(x)) \mathbf{e}
\end{equation}
for $\tau > 0$.
Under bounding Assumptions \ref{as: bounded} and \ref{as: noize} $\alpha$-th moment in expectation will be bounded as well.
\begin{lemma}\label{lem: grad 1 + k norm intro}
Under Assumptions \ref{as: bounded} and \ref{as: noize}, for $q \in [2, +\infty)$, we have
$$\E_{\xi, \mathbf{e}}[||g_t(x,\xi,\mathbf{e})||_q^{\alpha+1}]\leq 2^{\alpha}\left(\frac{na_{q}B}{\tau}\right)^{\alpha+1} + 2^{\alpha}\left(\frac{na_{q}\Delta}{\tau}\right)^{\alpha+1}  := \sigma_{q}^{\alpha+1},$$
where $a_q := n^{\frac1q - \frac12}  \min \{ \sqrt{32\ln n - 8} , \sqrt{2q - 1}\}.$
\end{lemma}

\subsection{Nonlinear Clipping Algorithm}

Now we briefly describe the Stochastic Mirror Descent (SMD) algorithm which is the base of our clipping algorithm. Let the function $\psi: \R^n \rightarrow \R$ called the prox-function be $1$-strongly convex w.r.t. the $p$-norm and continuously differentiable. We denote its Fenchel conjugate and its Bregman divergence respectively as
$$\psi^*(y)= \sup\limits_{x \in \R^n} \{\langle x,y \rangle - \psi(x) \} \quad \text{and} \quad  B_{\psi}(x,y) = \psi(x) - \psi(y) - \langle \nabla \psi(y), x -y \rangle.$$
The Stochastic Mirror Descent updates with stepsize $\mu$ and update vector $g_{t+1}$ are as follows:
\begin{equation}\label{MD}
y_{t+1} = \nabla(\psi^*) (\nabla \psi(x_t) - \mu g_{t+1}), \quad x_{t+1} = \arg \min\limits_{x \in \S} B_{\psi}(x, y_{t+1}).
\end{equation}

The final algorithm has the following structure. It is a modification of the Zeroth-order Clipping Algorithm from  \cite{kornilov2023gradient}.

\begin{algorithm}[H]
\caption{Nonlinear Bandits Clipping Algorithm }\label{alg:clip}
\begin{algorithmic}[1]
\REQUIRE{Number of iterations $T$, stepsize $\mu$, clipping constant $\lambda$, prox-function $\psi_{p}$, smoothing constant $\tau$}
    \STATE $x_1 \gets \arg\min\limits_{x \in \S} \psi_{p}(x)$
    
    \FOR{$t = 1, \dots ,T$}  
    
        \STATE Sample $\mathbf{e}_t \sim Uniform(\{\mathbf{e}: \|\mathbf{e}\|_2 = 1 \}) $ independently 
        
        \STATE Calculate  sequence's point $z_t = x_t + \tau \mathbf{e}_t$
        \STATE Get loss $\phi_t(z_t, \xi_t)$
        \STATE Calculate ${g}_{t} = \frac{n}{\tau}(\phi_t(z_t, \xi_t)) \mathbf{e}_t$
         
        \STATE Calculate clipped $\hat{g}_{t} = \frac{g_{t}}{||g_{t}||_q} \min(||g_{t}||_q, \lambda)$ 
        
        \STATE Calculate $y_{t+1}  \gets \nabla(\psi_{p}^*) (\nabla \psi_{p}(x_{t}) - \mu \hat{g}_{t})$
        
        \STATE Calculate $x_{t+1}  \gets \arg \min\limits_{x \in \S} B_{\psi_p}(x, y_{t+1})$

    \ENDFOR

    \RETURN $\{z_t\}_{t=1}^T$

\end{algorithmic}
\end{algorithm}

Our main theorem gives an estimation in terms of the expectation of the average pseudo-regret for the sequence generated by the algorithm described above.

\begin{theorem}\label{Clip Conv} Let the functions $l_t, \delta_t$ satisfying Assumptions~\ref{as: convex}, \ref{as: bounded}, \ref{as: noize} and one of Assumptions~\ref{as: Lipshcitz and bounded} or \ref{as: Smooth and bounded},  $q \in [2, \infty]$, an arbitrary number of iterations $T$, and a smoothing constant $\tau > 0$  be given. Choose a $1$-strongly convex w.r.t. the $p$-norm prox-function $\psi_{p}(x)$. Set the stepsize $\mu = \left( \frac{R_1^2}{4T\sigma_{q}^{\alpha+1} \mathcal{D}_\psi^{1-\alpha} }\right)^{\frac{1}{\alpha+1}}$ with $\sigma_{q}$ given in Lemma~\ref{lem: grad 1 + k norm}, distanse between $x_1$ and solution $x^*$,  $R_1^{\frac{\alpha+1}{{\alpha}}} =  \frac{\alpha+1}{{\alpha}} B_{\psi_p}(x^*, x_1)$ and diameter $\mathcal{D}_{\psi}^\frac{\alpha+1}{{\alpha}} = \frac{\alpha+1}{{\alpha}}   B_{\psi_p}(x,y).$ Then we set the clipping constant $\lambda =  \frac{2{\alpha}\mathcal{D}_\psi}{(1-\alpha)\mu}$.
Let $\{z_t\}_{t=1}^T$ be the sequence generated by Algorithm \ref{alg:clip} with the above parameters.
\begin{enumerate}
    \item  Then, \begin{equation}\label{eq: clip theorem eq intro}    
\frac1T \E[ \mathcal{R}_T(\{l_t(\cdot)\}, \{z_t\})] \leq 4M\tau + \Delta \frac{\sqrt{n}}{\tau} \mathcal{D}_\psi+4 \frac{R_1^\frac{2{\alpha}}{\alpha+1} \mathcal{D}_\psi^{\frac{1-\alpha}{\alpha+1}}n a_q (\Delta + B) }{\tau T^\frac{{\alpha}}{\alpha+1}}\end{equation}
for Assumption~\ref{as: Lipshcitz and bounded},
\begin{equation}\label{eq: clip theorem eq smooth intro}    
\frac1T \E[ \mathcal{R}_T(\{l_t(\cdot)\}, \{z_t\})] \leq 2 L\tau^2 + \Delta \frac{\sqrt{n}}{\tau} \mathcal{D}_\psi+4 \frac{R_1^\frac{2{\alpha}}{\alpha+1} \mathcal{D}_\psi^{\frac{1-\alpha}{\alpha+1}}n a_q (\Delta + B) }{\tau T^\frac{{\alpha}}{\alpha+1}}\end{equation}
for Assumption~\ref{as: Smooth and bounded}.
\item Let $\varepsilon$ be the desired average pseudo-regret accuracy in the problem with no adversarial noise $\Delta = 0$. For Assumption~\ref{as: Lipshcitz and bounded} with $\tau_M  = \frac{\varepsilon}{8M}$ the number of iterations to achieve the desired accuracy is as follows:
$$T_M = \left( \frac{8M\left(4R_1^\frac{2{\alpha}}{\alpha+1} \mathcal{D}_\psi^{\frac{1-\alpha}{\alpha+1}}n a_q  B\right)}{\varepsilon^2}\right)^\frac{\alpha+1}{\alpha}= O\left(\left(\frac{MB n a_q \mathcal{D}_\psi}{\varepsilon^2}\right)^\frac{\alpha+1}{\alpha}\right).$$
For Assumption~\ref{as: Smooth and bounded} with $\tau_L  = \sqrt{\frac{\varepsilon}{4L}}$ the number of iterations to achieve the desired accuracy is
$$T_L = \left( \frac{\sqrt{4L}\left(4R_1^\frac{2{\alpha}}{\alpha+1} \mathcal{D}_\psi^{\frac{1-\alpha}{\alpha+1}}n a_q  B\right)}{\varepsilon^\frac32}\right)^\frac{\alpha+1}{\alpha } = O\left(\left(\frac{\sqrt{L}B n a_q \mathcal{D}_\psi}{\varepsilon^\frac32}\right)^\frac{\alpha+1}{\alpha}\right).$$
\end{enumerate}
\end{theorem}

All proofs can be found in section \ref{sec: proofs}.

For the standard simplex  $\triangle_n = \{x \in \R^n: x \geq 0, \sum_i x_i  = 1 \}$ the optimal choice of norm and prox-function is $p = 1$ and $\psi_p = (1+\gamma)\sum_{i=1}^n (x_i + \gamma/d)\log(x_i + \gamma/d), \gamma > 0$, respectively. In this case, $p$ and $\psi_p$-depending $n a_q \mathcal{D}_\psi  $ equals $ O(\sqrt{n}\log n)$.

\section{Numerical Experiments}\label{sec:Exp}
We conducted experiments to demonstrate the superior performance of our INF-clip algorithm in specific stochastic Multi-Armed Bandit (MAB) scenarios with heavy tails  compared to HTINF, Robust UCB, and APE. To showcase this, we focus on an experiment involving only two available arms ($n=2$). Each arm $i$ generates a random losses $g_{t, i} \sim \xi \cdot \beta_{i}$. Here, the parameters $\beta_{0} = \frac{3}{\mathbf{E\left(\xi\right)}}, \beta_{1} = \frac{3.1}{\mathbf{E\left(\xi\right)}}$ are fixed and the random variable $\xi$ has $\text{pdf}_{\xi}\left(x, \alpha\right) = \frac{C}{x^{2 + \alpha} \cdot ln(x) ^ {2}}$ in domain $x \in \left[2; +\infty \right)$, where $C$ is a normalization constant.

In this experimental setup, individual experiments are subject to significant random deviations. To enhance the informativeness of the results, we conduct 100 individual experiments and analyze the aggregated statistics.

It is well-known that in a 2-arms setting, the stochastic MAB problem is equivalent to the best arm identification problem. Thus, instead of regret, we aim to assess the probability of selecting the best arm based on the observed history.

By design, we possess knowledge of the conditional probability of selecting the optimal arm for all algorithms, which remains stochastic due to the nature of the experiment's history.

To mitigate the high dispersion in APE probabilities, we apply an average filter with a window size of 30 to reduce noise in the plot.

\begin{figure}[htp]
    \centering
    \includegraphics[width=12cm, height=8cm]{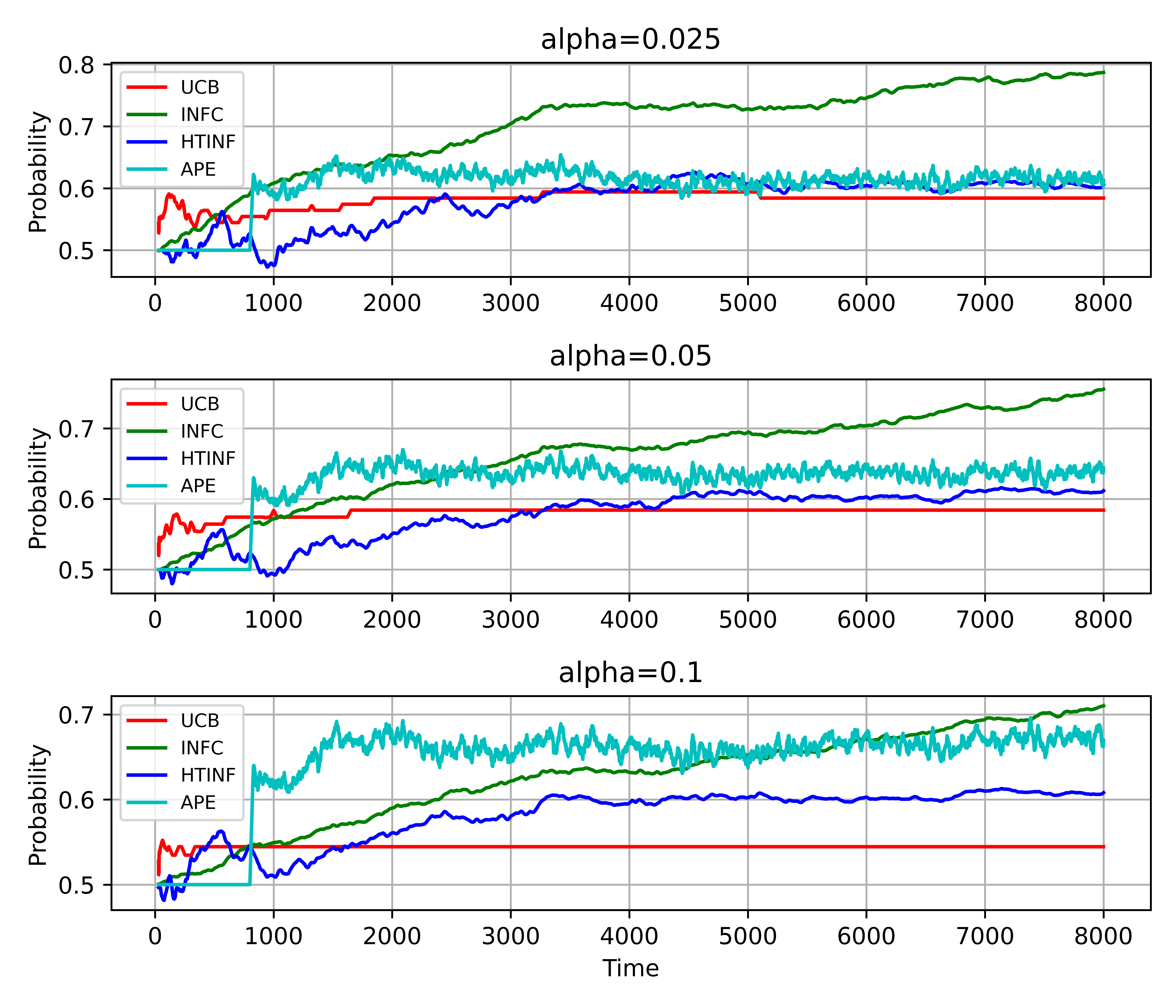}
    \caption{Probability of optimal arm picking mean (aggregated with average filter) for 100 experiments and 8000 iterations with low alphas}
    \label{fig:galaxy}
\end{figure}
\begin{figure}[htp]
    \centering
    \includegraphics[width=12cm, height=5cm]{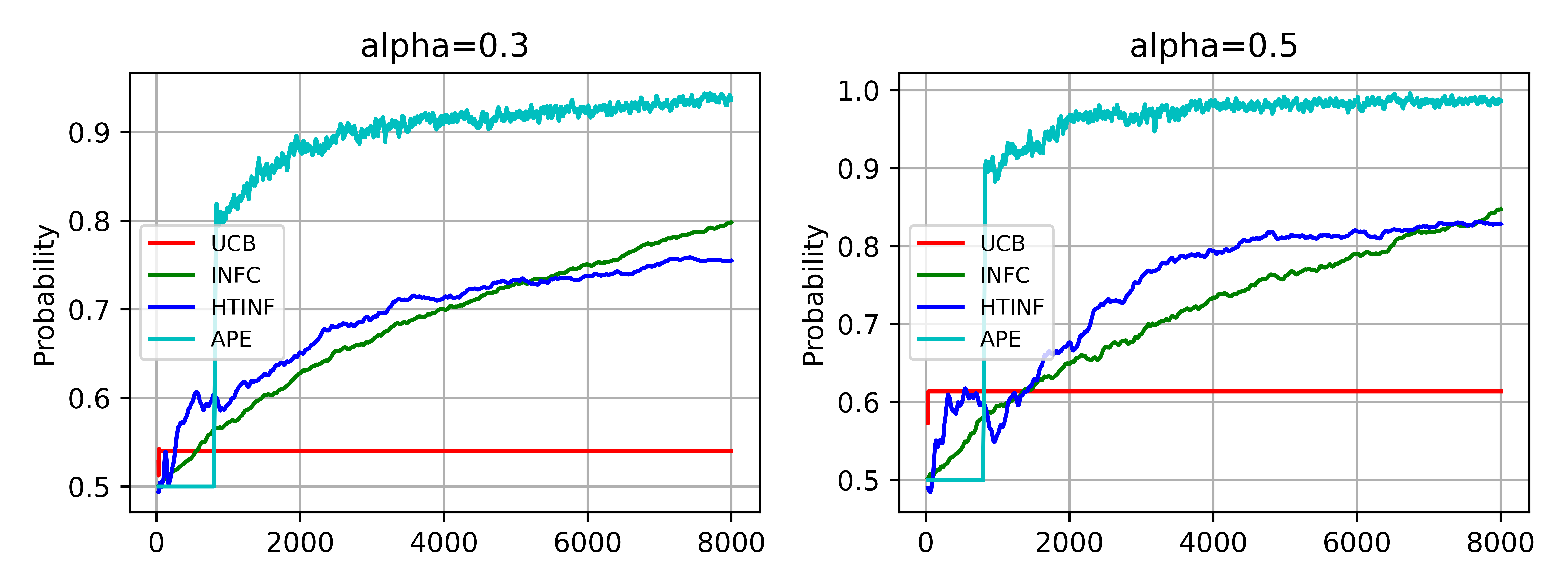}
    \caption{Probability of optimal arm picking mean (aggregated with average filter) for 100 experiments and 8000 iterations with high alphas}
    \label{fig:galaxy2}
\end{figure}
\begin{figure}[htp]
    \centering
    \includegraphics[width=12cm, height=8cm]{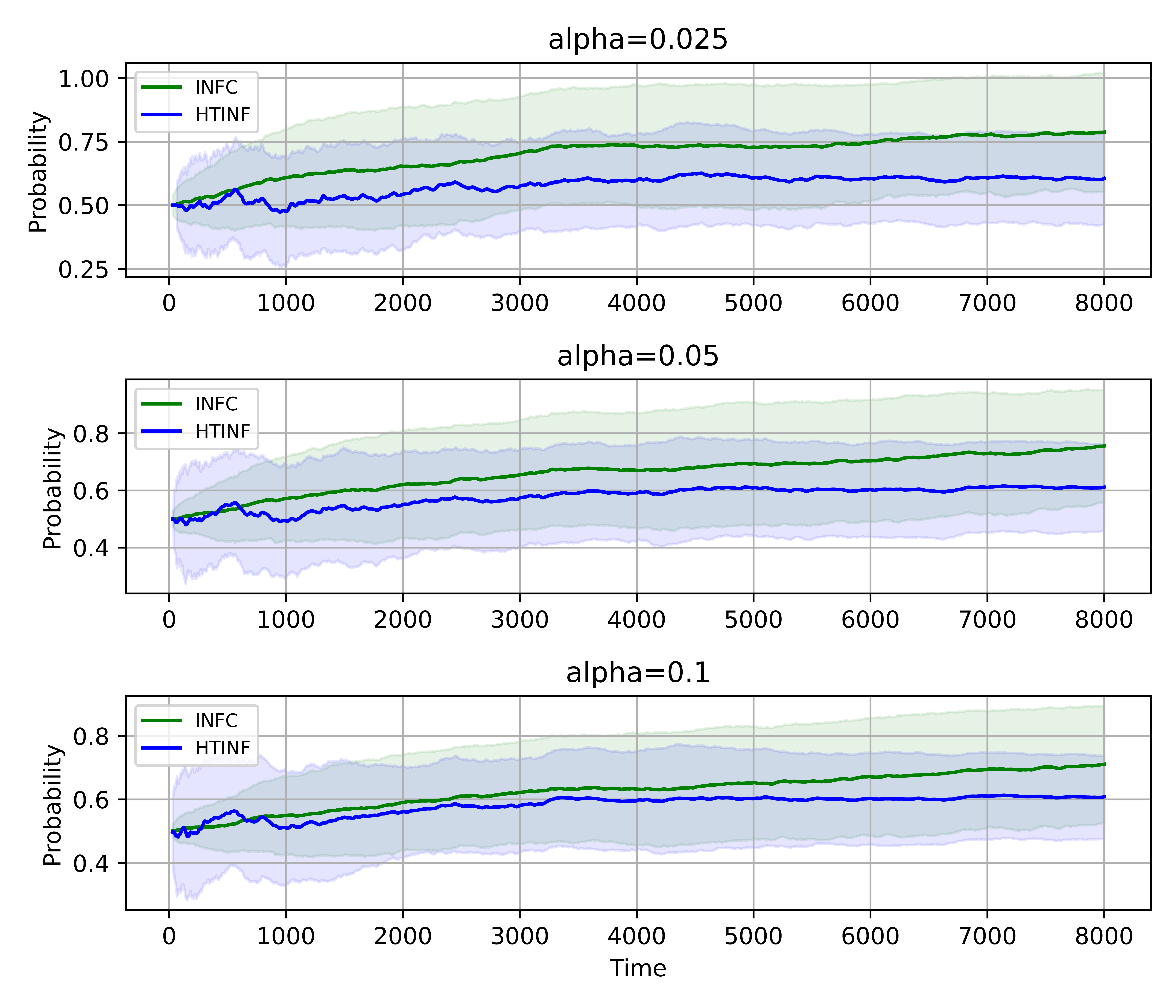}
    \caption{Probability of optimal arm picking mean (aggregated with average filter) for 100 experiments and 8000 iterations with high alphas with ± std bounds for probabilistic methods}
    \label{fig:galaxy3}
\end{figure}
\begin{figure}[htp]
    \centering
    \includegraphics[width=12cm, height=5cm]{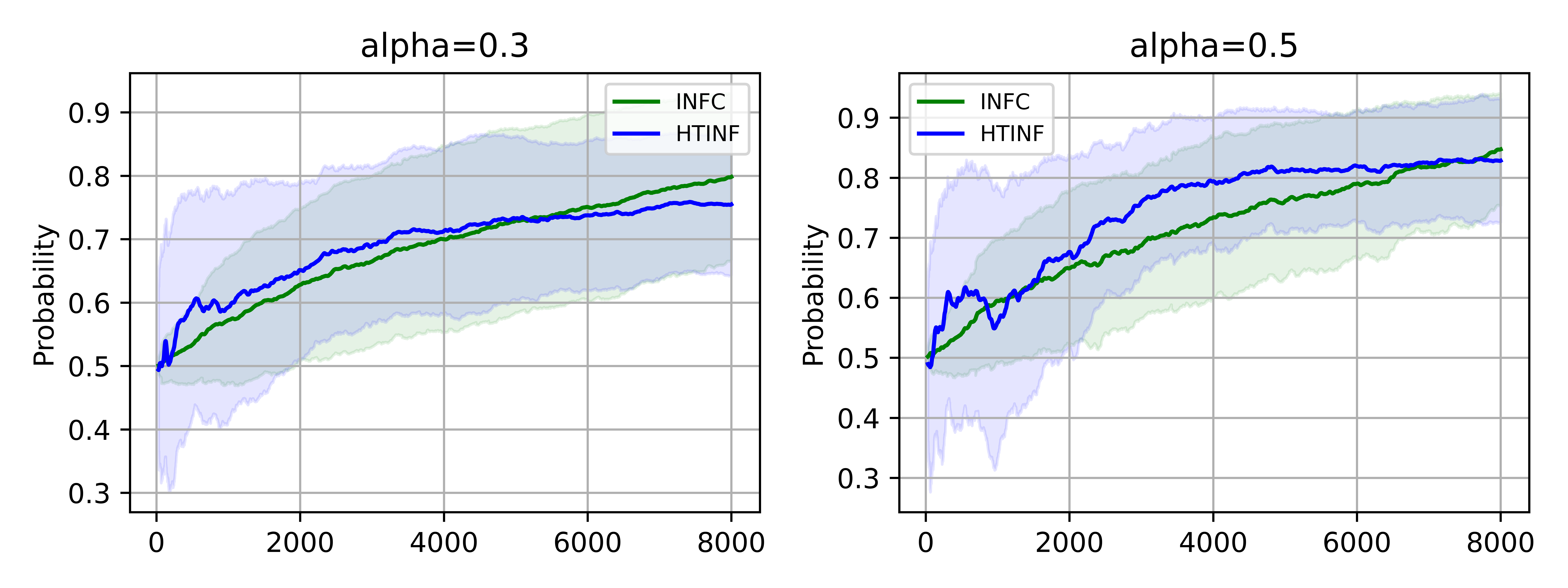}
    \caption{Probability of optimal arm picking mean (aggregated with average filter) for 100 experiments and 8000 iterations with high alphas with ± std bounds for probabilistic methods}
    \label{fig:galaxy4}
\end{figure}

The results presented in Fig.\ref{fig:galaxy}. As we can see, lower values of $\alpha$ (i.e., "heavier" tails) lead to a greater gap in performance of INF-clip compared to HTINF and other algorithms. This is natural, because HTINF truncates the tail of the distribution and does not use outliers in best arm identification, but our algorithm does exactly this. So the heavier the tail, the more samples contribute only to the INF-clip. After reaching some alpha threshold, APE performs better than INF-clip, probably due to a large exploration phase.

\subsection{Reproducibility}

Interested readers can find the source code for the experiments here:

https://github.com/Kutuz4/ImplicitlyNormalizedForecasterWithClipping
%\newline

%\begin{figure}[htp]
 %   \centering
  %  \includegraphics[width=10cm, height=5cm]{expected_reget_evolution_5000.png}
   % \caption{Expected regret mean with ± std bounds for 100 experiments and 5000 iterations}
    %\label{fig:galaxy}
%\end{figure}
%We can see that HTINF algorithm works better in the beginning, but the difference between HTINF and INF-clip decreases over time, and, for $\alpha$=0.1 on 5000 iterations INF-clip works better and richs better expected regret value at the end. But, INF-clip starts to works better in terms of optimal arm picking probability much more earlier. This advantage is shown on Figure 2.

%This experiments show, that our algorithm has longer expolarion period, but is better than HTINF on long distances. We can suggest such explanation of this phenomenon: HTINF threshold for ignoring sample is bigger than our clipping threshold, so it can capture full information for bigger loss values. But, despite to INF-clip, it ignores all information about big loss values. So, there are more chances to capture bigger amount information at start, but after that, due to policy of ignoring large values, HTINF convergence rate is much more slower.

\section{Missing proofs of nonlinear bandits section}\label{sec: proofs}

\textbf{Assumptions lemmas}

For begining we prove basic inequalities based on Assumptions. 
\begin{lemma}\label{lem: Jensen for norm}

\begin{enumerate}
    \item 

    For all $x,y \in \R^{n}$ and $\alpha \in (0,1]$: \begin{equation}\label{inq}
    ||x-y||_q^{\alpha+1} \leq 2^{\alpha }||x||_q^{\alpha+1} + 2^{\alpha}||y||_q^{\alpha+1},
\end{equation}
\item 
\begin{equation}\label{jensen norm less 1}
\forall x,y \in \R , x,y \geq 0, {\kappa} \in [0,1]: (x-y)^{\kappa} \leq x^{\kappa} + y^{\kappa}.
\end{equation}
\end{enumerate}
\end{lemma}

\begin{proof}
    \begin{itemize}
        \item For \eqref{inq} by Jensen's inequality for convex $||\cdot||_q^{\alpha+1}$ with $\alpha> 0$
        $$||x-y||_q^{\alpha+1} = 2^{\alpha+1}||x/2-y/2||_q^{\alpha+1} \leq2^{\alpha}||x||_q^{\alpha+1} + 2^{\alpha}||y||_q^{\alpha+1}. $$
        \item \eqref{jensen norm less 1} is the proposition $9$ from \cite{vural2022mirror}.
    \end{itemize}

\end{proof}
\begin{lemma}\label{lem: Lipschitz f }
    Assumption~\ref{as: Lipshcitz and bounded} implies that $f(x)$ is $M$ Lipschitz on $\S$.
\end{lemma}
\begin{proof}
    For all $x,y \in \S$
    \begin{eqnarray}|f(x) - f(y)| &=& |\E[f(x,\xi) - f(y,\xi)]| \overset{\text{Jensen's inq}}{\leq}  \E[|f(x,\xi) - f(y,\xi)|] \notag \\
    &\leq& \E[M(\xi)]||x-y||_2 \notag \\
    &\overset{\text{Jensen's inq}}{\leq}& \E[M(\xi)^{\alpha+1}]^\frac{1}{\alpha+1} ||x-y||_2 \leq M ||x-y||_2. \notag\end{eqnarray}
\end{proof}

\begin{lemma}\label{lem: Smooth f }
    Assumption~\ref{as: Smooth and bounded} implies that $f(x)$ is $L$ Smooth on $\S$.
\end{lemma}
\begin{proof}
    For all $x,y \in \S$
    \begin{eqnarray}|\nabla f(x) - \nabla f(y)| &=& |\E[\nabla f(x,\xi) - \nabla f(y,\xi)]|  \notag \\&\overset{\text{Jensen's inq}}{\leq}&  \E[|\nabla f(x,\xi) - \nabla f(y,\xi)|] \notag \\
    &\leq& \E[L(\xi)]||x-y||_2 \notag \\
    &\overset{\text{Jensen's inq}}{\leq}& \E[L(\xi)^{\alpha+1}]^\frac{1}{\alpha+1} ||x-y||_2 \leq L ||x-y||_2. \notag\end{eqnarray}
\end{proof}

\textbf{Gradient-free lemmas}

The next lemma gives estimates for the quality of the approximation function $\hat{l}_t^\tau(x)$. The proof can be found in \cite[ Theorem $2.1$]{gasnikov2022power}.
\begin{lemma}\label{lem: hat_f properties intro}
Let Assumption~\ref{as: convex} holds.
\begin{enumerate}
\item If Assumption~\ref{as: Lipshcitz and bounded} holds too then function $\hat{l}_t^\tau(x)$ is convex, Lipschitz with constant $M$  on $\S$, and satisfies
$$\sup \limits_{x \in \S} |\hat{l}_t^\tau(x) - l_t(x)| \leq \tau M.$$

\item If Assumption~\ref{as: Smooth and bounded} holds too then function $\hat{l}_t^\tau(x)$ is convex, Smooth with constant $L$  on $\S$, and satisfies
$$\sup \limits_{x \in \S} |\hat{l}_t^\tau(x) - l_t(x)| \leq \frac{L\tau^2}{2}.$$

    \item 
Function $\hat{l}_t^\tau(x)$ is differentiable on $\S$ with the following gradient
$$\nabla \hat{l}_t^\tau (x) = \E_\mathbf{e}\left[\frac{d}{\tau} l_t(x + \tau \mathbf{e}) \mathbf{e}\right].$$

\end{enumerate}
\end{lemma}

The following lemma gives some useful facts about the measure concentration on the Euclidean unit sphere.
\begin{lemma}\label{upper bounds}
For $q \geq 2, \alpha \in (0,1]$  
$$\E_\mathbf{e} \left[||\mathbf{e}||_q^{2(\alpha+1)}\right]  \leq a_{q}^{2(\alpha+1)} = n^{\frac1q - \frac12}  \min \{ \sqrt{32\ln n - 8} , \sqrt{2q - 1}\}.$$. 
\end{lemma}
This Lemma is a generalization of the Lemma from \cite{gorbunov2019upper} for $\alpha < 1$. Proof can be found in \cite{kornilov2023gradient}.

\begin{lemma}\label{inner product estimate}
For the random  vector $\mathbf{e}$ uniformly distributed on the Euclidean sphere $\{\mathbf{e} \in \R^n: ||\mathbf{e}||_2 = 1\}$ and for any $r \in \R^n$, we have
$$\E_\mathbf{e}[|\langle \mathbf{e}, r \rangle|] \leq \frac{||r||_2}{\sqrt{n}}.$$
\end{lemma}
\begin{lemma}\label{inner product grad r}
Let $g_t(x, \xi, \mathbf{e})$  be defined in \eqref{g} and $\hat{l}_t^\tau(x)$ be defined in \eqref{hat_f}. Then, the following holds under Assumption \ref{as: noize}:
$$\E_{\xi, \mathbf{e}} [\langle g_t(x,\xi, \mathbf{e}), r \rangle] \geq \langle \nabla \hat{l}_t^\tau(x) , r \rangle - \frac{n \Delta}{\tau}\E_\mathbf{e} [|\langle \mathbf{e}, r \rangle|] $$
for any $r \in \R^n.$
\end{lemma}
\begin{proof}
By definition

$$g_t(x, \xi, \mathbf{e}) = \frac{n}{\tau}(l_t(x + \tau \mathbf{e}, \xi) + \delta_t(x + \tau \mathbf{e})) \mathbf{e}.$$
Then
$$\E_{\xi, \mathbf{e}} [\langle g_t(x,\xi, \mathbf{e}), r \rangle] = \frac{n}{\tau}\E_{\xi, \mathbf{e}} [\langle l_t(x + \tau \mathbf{e}, \xi) \mathbf{e}, r \rangle]  + \frac{n}{\tau}\E_{\xi, \mathbf{e}} [\langle \delta_t(x + \tau \mathbf{e}) \mathbf{e}, r \rangle]. $$
In the first term
$$\frac{n}{\tau}\E_{\xi, \mathbf{e}} [\langle (l_t(x + \tau \mathbf{e}, \xi) )\mathbf{e}, r \rangle]  =\frac{n}{\tau}\E_{ \mathbf{e}} [\langle \E_{\xi}[l_t(x + \tau \mathbf{e}, \xi)] \mathbf{e}, r \rangle] $$
$$= \frac{n}{\tau} \langle \E_{ \mathbf{e}} [l_t(x + \tau \mathbf{e}) \mathbf{e}], r \rangle.$$
Using Lemma \ref{lem: hat_f properties intro}
$$ \frac{n}{\tau} \langle \E_{ \mathbf{e}} [l_t(x + \tau \mathbf{e}) \mathbf{e}], r \rangle = \langle \nabla \hat{l}_t^\tau(x), r \rangle.$$
In the second term we use Assumption \ref{as: noize}
$$\frac{n}{\tau}\E_{\xi, \mathbf{e}} [\langle \delta_t(x + \tau \mathbf{e})\mathbf{e}, r \rangle] \geq - \frac{n\Delta}{\tau} \E_\mathbf{e}[|\langle \mathbf{e}, r \rangle|].$$
Adding two terms together we get the necessary result.

\end{proof}

Finally, a lemma on the boundedness $(\alpha+1)$-th moment of the estimated gradient.
\begin{lemma}\label{lem: grad 1 + k norm}
Under Assumptions \ref{as: bounded} and \ref{as: noize}, for $q \in [2, +\infty)$, we have
$$\E_{\xi, \mathbf{e}}[||g_t(x,\xi,\mathbf{e})||_q^{\alpha+1}]\leq 2^{\alpha}\left(\frac{na_{q}B}{\tau}\right)^{\alpha+1} + 2^{\alpha }\left(\frac{na_{q}\Delta}{\tau}\right)^{\alpha+1}  = \sigma_{q}^{\alpha+1},$$
where $a_q = n^{\frac1q - \frac12}  \min \{ \sqrt{32\ln n - 8} , \sqrt{2q - 1}\}.$
\end{lemma}
\begin{proof}
\begin{eqnarray}
\E_{\xi, \mathbf{e}}[||g_t(x,\xi,\mathbf{e})||_q^{\alpha+1}] &=& \E_{\xi, \mathbf{e}}\left[\left|\left|\frac{n}{\tau}\phi_t(x + \tau \mathbf{e}, \xi) \mathbf{e}\right|\right|_q^{\alpha+1}\right]  \notag\\
&\leq& \left(\frac{n}{\tau} \right)^{\alpha}\E_{\xi, \mathbf{e}}\left[||\mathbf{e}||_q^{\alpha+1}|(l_t(x + \tau \mathbf{e}, \xi)  + \delta_t(x + \tau \mathbf{e}) ) |^{\alpha+1}\right]  \notag\\
&\leq& 2^{\alpha}\left(\frac{n}{\tau} \right)^{\alpha+1}\E_{\xi, \mathbf{e}}\left[||\mathbf{e}||_q^{\alpha+1}|l_t(x + \tau \mathbf{e}, \xi) |^{\alpha+1}\right] \label{eq: grad 1 +k norm first term}\\
&+& 2^{\alpha}\left(\frac{n}{\tau} \right)^{\alpha+1} \E_{\xi, \mathbf{e}}\left[||\mathbf{e}||_q^{\alpha+1}|\delta_t(x + \tau \mathbf{e}) |^{\alpha+1}\right].\label{eq: grad 1 +k norm second term} 
\end{eqnarray}
Lets deal with \eqref{eq: grad 1 +k norm first term} term. We use Cauchy-Schwartz inequality, Assumption \ref{as: bounded}  and $\E_\mathbf{e} \left[||\mathbf{e}||_q^{2(\alpha+1)}\right]  \leq a_{q}^{2(\alpha+1)}$ by definition

\begin{eqnarray}\E_{\xi, \mathbf{e}}\left[||\mathbf{e}||_q^{\alpha+1}|l_t(x + \tau \mathbf{e}, \xi)  |^{\alpha+1}\right] \notag \\
 \leq \sqrt{\E_\mathbf{e} \left[||\mathbf{e}||_q^{2(\alpha+1)}\right]  \E_\mathbf{e} \left[|l_t(x + \tau \mathbf{e}, \xi)  |^{2(\alpha+1)}\right]} \notag \\
\leq a_{q}^{\alpha+1}B^{\alpha+1} = (a_{q}B)^{\alpha+1}.\label{eq: grad 1+k norm final first term}
\end{eqnarray}
Lets deal with \eqref{eq: grad 1 +k norm second term} term. We use Cauchy-Schwartz inequality, Assumption \ref{as: noize}  and $\E_\mathbf{e} \left[||\mathbf{e}||_q^{2(\alpha+1)}\right]  \leq a_{q}^{2(\alpha+1)}$ by definition

\begin{eqnarray}\E_{\xi, \mathbf{e}}\left[||\mathbf{e}||_q^{\alpha+1}|\delta_t(x + \tau \mathbf{e})  |^{\alpha+1}\right] \notag \\
 \leq \sqrt{\E_\mathbf{e} \left[||\mathbf{e}||_q^{2(\alpha+1)}\right]  \E_\mathbf{e} \left[|\delta_t(x + \tau \mathbf{e})  |^{2(\alpha+1)}\right]} \notag \\
\leq a_{q}^{\alpha+1}\Delta^{\alpha+1} = (a_{q}\Delta)^{\alpha+1}.\label{eq: grad 1+k norm final second term}
\end{eqnarray}

Adding \eqref{eq: grad 1+k norm final first term} and \eqref{eq: grad 1+k norm final second term} we get the final result
$$\E_{\xi, \mathbf{e}}[||g_t(x,\xi,\mathbf{e})||_q^{\alpha+1}] \leq 2^{\alpha}\left(\frac{na_{q}B}{\tau}\right)^{\alpha+1} + 2^{\alpha}\left(\frac{na_{q}\Delta}{\tau}\right)^{\alpha+1}.  $$

\end{proof}

\textbf{Proof of Clipping Algorithm in Expectation Convergence}

The clipped gradient has bunch of useful properties for further proofs. Proof can be found in \cite{kornilov2023gradient}.

\begin{lemma}\label{clip_grad_properties}
For $\lambda >0$ we define $\hat{g} = \frac{g}{||g||_q} \min(||g||_q, \lambda)$.
\begin{enumerate}
    \item 
    \begin{equation}||\hat{g} - \E[\hat{g} ]||_q \leq 2\lambda.\end{equation}
    
    \item 
If $\E[||g||_q^{\alpha+1}]\leq  \sigma_{q}^{\alpha+1}$, then \begin{enumerate}
    \item \begin{equation}\E[||\hat{g}||^2_q] \leq  \sigma_{q}^{\alpha+1} \lambda^{1 - \alpha}.\end{equation}
    \item \begin{equation}\E[||\hat{g} - \E[\hat{g}]||^2_q] \leq  4\sigma_{q}^{\alpha+1} \lambda^{1 - \alpha}.\end{equation}
    \item \begin{equation}\label{eq: bayes clip and not}||\E[g] - \E[\hat{g}]||_q  \leq \frac{\sigma_{q}^{\alpha+1}}{\lambda^{\alpha}}.\end{equation}
    
\end{enumerate}
\end{enumerate}
\end{lemma}

\begin{theorem}\label{Clip Conv app} Let the functions $l_t, \delta_t$ satisfying Assumptions~\ref{as: convex}, \ref{as: bounded}, \ref{as: noize} and one of the Assumptions~\ref{as: Lipshcitz and bounded} or \ref{as: Smooth and bounded},  $q \in [2, \infty]$, an arbitrary number of iterations $T$, and a smoothing constant $\tau > 0$,  be given. Choose $1$-strongly convex w.r.t. the $p$-norm prox-function $\psi_{p}(x)$. Set the stepsize $\mu = \left( \frac{R_1^2}{4T\sigma_{q}^{\alpha+1} \mathcal{D}_\psi^{1-\alpha} }\right)^{\frac{1}{\alpha+1}}$ with $\sigma_{q}$ given in Lemma~\ref{lem: grad 1 + k norm}, the distanse between $x_1$ and the solution $x^*$, $R_1^{\frac{\alpha+1}{{\alpha}}} =  \frac{\alpha+1}{{\alpha}} B_{\psi_p}(x^*, x_1)$ and diameter $\mathcal{D}_{\psi}^\frac{\alpha+1}{{\alpha}} = \frac{\alpha+1}{{\alpha}}   B_{\psi_p}(x,y).$ After set the clipping constant $\lambda =  \frac{2{\alpha}\mathcal{D}_\psi}{(1-\alpha)\mu}$.
Let $\{z_t\}_{t=1}^T$ be a sequence generated by Algorithm \ref{alg:clip} with the above parameters.
Then,

\begin{equation}\label{eq: clip theorem eq}    
\frac1T \E[ \mathcal{R}_T(\{l_t(\cdot)\}, \{z_t\})] \leq 4M\tau + \Delta \frac{\sqrt{n}}{\tau} \mathcal{D}_\psi+4 \frac{R_1^\frac{2{\alpha}}{\alpha+1} \mathcal{D}_\psi^{\frac{1-\alpha}{\alpha+1}}n a_q (\Delta + B) }{\tau T^\frac{{\alpha}}{\alpha+1}}\end{equation}
for Assumption~\ref{as: Lipshcitz and bounded},

\begin{equation}\label{eq: clip theorem eq smooth}    
\frac1T \E[ \mathcal{R}_T(\{l_t(\cdot)\}, \{z_t\})] \leq 2 L\tau^2 + \Delta \frac{\sqrt{n}}{\tau} \mathcal{D}_\psi+4 \frac{R_1^\frac{2{\alpha}}{\alpha+1} \mathcal{D}_\psi^{\frac{1-\alpha}{\alpha+1}}n a_q (\Delta + B) }{\tau T^\frac{{\alpha)}}{\alpha+1}}\end{equation}
for Assumption~\ref{as: Smooth and bounded}.
\end{theorem}
\begin{proof}
Everywhere below, we consider the point $u \in \S$ to be fixed. Also we denote for any moment $t$ the conditional expectation over the previous steps of the Algorithm
$$\E_{|\leq t}[\cdot] := \E[\cdot| x_t, x_{t-1}, \dots , x_1].$$

In the case of Lipschitz $l_t(x)$ we can get the next inequality. For all $t$ let's notice that $ \hat{l}^\tau_t$  is convex from Assumption~\ref{as: convex} and approximates $l_t$ from Lemma~\ref{lem: hat_f properties intro} 
     \begin{equation}\label{loss} l_t(x_t) - l_t(u)  \leq \hat{l}_t^\tau(x_t) - \hat{l}_t^\tau(u) + 2M\tau \leq \langle \nabla \hat{l}_t^\tau(x_{t}) , x_t  - u \rangle +  2M\tau. \end{equation}
In the case of Smooth $l_t(x)$ bound changes.
For any point $u \in \S $ and for all $t$ let's notice that $ \hat{l}^\tau_t$  is convex from Assumption~\ref{as: convex} and approximates $l_t$ from Lemma~\ref{lem: hat_f properties intro} 
     \begin{equation}\label{loss_2} l_t(x_t) - l_t(u)  \leq \hat{l}_t^\tau(x_t) - \hat{l}_t^\tau(u) + L\tau^2 \leq \langle \nabla \hat{l}_t^\tau(x_{t}) , x_t  - u \rangle +  L\tau^2. \end{equation}

In the following we will consider the Lipschitz case. In the smooth case one needs to replace $M\tau$ with $\frac{L\tau^2}{2}$.

At each step of Algorithm \ref{alg:clip} we calculate $g_t$
$${g}_{t} = \frac{n}{\tau}(\phi_t(z_t, \xi_t)) \mathbf{e}_t, \quad \hat{g}_t = \text{clip}(g_t, \lambda),$$
where $z_t = x_t +\tau \mathbf{e}_t.$

We define functions $h_k(x) = \langle \E_{|\leq t}[\hat{g}_{t}], x  - u\rangle. $ Note that $h_k(x)$ is convex for any $t$ and $\nabla h_k(x) = \E_{|\leq t}[\hat{g}_{t}]$. Therefore, the sampled estimation gradient is unbiased.
With this, we can rewrite \eqref{loss_2}
$$\langle \nabla \hat{l}_t^\tau(x_t) , x_t  - u \rangle +  2M\tau $$
\begin{equation}\label{D and E terms in clip}
    = \underbrace{ \langle \nabla \hat{l}_t^\tau(x_t) - \E_{|\leq t}[\hat{g}_{t}] , x_t  - u \rangle    }_\text{D}+  \underbrace{  h_t(x_t) - h_t(u)}_\text{E} + 2M\tau.\end{equation}
We bound D term by Lemma \ref{clip_grad_properties}
$$\E\left[\langle \nabla \hat{l}_t^\tau(x_t) - \E_{|\leq t}[\hat{g}_{t}] , x_t - u \rangle  \right] $$
\begin{equation}\label{eq: conv clip D term sum}
    = \E\left[\langle \nabla \hat{l}_t^\tau(x_t) - \E_{|\leq t}[g_{t}] , x_t - u \rangle  + \langle  \E_{|\leq t}[{g}_{t}] - \E_{|\leq t}[\hat{g}_{t}] , x_t  - u \rangle  \right].\end{equation}
To bound the first term in \eqref{eq: conv clip D term sum} let's notice that $\psi_{p}$ is $\left( 1, 2\right)$-uniformly convex function w.r.t. $p$ norm. Then by definition

$$||x_t - u||_{p} \leq \left(2 B_{\psi_p}(x_t, u) \right)^\frac12 \leq \sup_{x,y \in \S} \left(2 B_{\psi_p}(x, y) \right)^\frac12 = \mathcal{D}_\psi. $$
Hence, we estimate $||x_t - u||_{p} \leq \mathcal{D}_\psi.$

By Cauchy–Schwarz inequality 
$$\E\left[\langle  \E_{|\leq t}[{g}_{t}] - \E_{|\leq t}[\hat{g}_{t}] , x_t  -u \rangle   \right] $$
\begin{equation} \label{eq: conv clip part 1}
    \leq  \E \left[||\E_{|\leq t}[{g}_{t}] - \E_{|\leq t}[\hat{g}_{t}]||_q||x_t- u||_{p}\right]     \overset{\eqref{eq: bayes clip and not}}{\leq} \mathcal{D}_\psi\frac{\sigma_{q}^{\alpha+1}}{\lambda^{\alpha}}.\end{equation}
To bound the second term in \eqref{eq: conv clip D term sum} we use Lemma \ref{inner product grad r} and Lemma \ref{inner product estimate}
$$\E\left[\langle \nabla \hat{l}_t^\tau(x_t) - \E_{|\leq t}[{g}_{t}] , x_t  -  u \rangle   \right]  $$
$$\leq  \frac{n \Delta}{\tau}\E\left[\E_{|\leq t}[|\langle \mathbf{e}_t,x_t - u\rangle|] \right]$$
$$ \leq \frac{n \Delta}{\tau} \frac{1}{\sqrt{d}}\E[||x_t - u||_2]$$
\begin{equation}\label{eq: conv clip part 2}
    \overset{p\leq 2}{\leq} \frac1T \sum \limits_{k=0}^{T-1} \frac{n \Delta}{\tau} \frac{1}{\sqrt{n}}\E[||x_t - u||_p] \leq \frac{\Delta \sqrt{n}}{\tau} \mathcal{D}_\psi.\end{equation}
Next, we bound E term 
$$\E \left[ h_t(x_t) - h_t(u)\right] \leq  \E \left[\E_{|\leq t}[\langle \E_{|\leq t}[\hat{g}_{t}], x_t  - u \rangle]\right] $$
$$\leq  \E \left[ \E_{|\leq t}[\langle \hat{g}_{t}, x_t  -u \rangle] \right].$$ 
For one step of SMD algorithm with $\hat{g_t}  $ (formula can be found in \cite{ben2001lectures}, for example) 

\begin{equation}\label{conv from E clip proof}
    \mu\langle \hat{g_t}, x_t - u\rangle \leq \frac{\mu^2}{2} ||\hat{g_t}||_q^2 +B_{\psi_p}(u, x_t) - B_{\psi_p}(u, x_{t+1}).
    \end{equation}
Using \eqref{conv from E clip proof} with application of $\E $ to both sides
$$  \E[\langle \hat{g}_{t}, x_t  - u \rangle] \leq  \frac{ \E [B_{\psi_p}(u, x_t) - B_{\psi_p}(u, x_{t+1})]  }{\mu} + \frac{\mu}{2}  \E\left[\E_{|\leq t}[||\hat{g}_{t}||^{2}_q]\right].$$
By Lemma \ref{clip_grad_properties}

$$\E_{|\leq t}(||\hat{g}_{t}||^2_q)  \leq \sigma_{q}^{\alpha+1} \lambda^{1-\alpha}. $$
Hence,
\begin{eqnarray}
    \label{eq: conv clip part 3}
    \E[\langle \hat{g}_{t}, x_t  - u \rangle] \leq  \frac{ \E [B_{\psi_p}(u, x_t) - B_{\psi_p}(u, x_{t+1})]  }{\mu} + \frac{\mu}{2}  \sigma_{q}^{\alpha+1} \lambda^{1-\alpha}. \end{eqnarray} 
Combining together equations \eqref{eq: conv clip part 1}, \eqref{eq: conv clip part 2}, \eqref{eq: conv clip part 3}, we get
$$\E [l_t(x_t)] - l_t(u)  \leq 2M\tau +\frac{ \E [B_{\psi_p}(u, x_t) - B_{\psi_p}(u, x_{t+1})]  }{\mu} $$
$$+ \frac{\mu}{2}  \sigma_{q}^{\alpha+1} \lambda^{1-\alpha} + \left(\frac{\sigma_{q}^{\alpha}}{\lambda^{(\alpha-1)}} + \Delta \frac{\sqrt{n}}{\tau}\right) \mathcal{D}_\psi. $$

Since we calculate Regret with $\{z_t\}$ sequence we use Lipshcitz Assumption \ref{as: Lipshcitz and bounded} and Lemma \ref{lem: Lipschitz f }
$$|l_t(z_t) - l_t(x_t)| \leq M ||z_t - x_t||_2 \leq M\tau ||\mathbf{e}_t||_2 = M\tau.$$
In Smooth case we analogically use Smooth Assumption \ref{as: Smooth and bounded} and Lemma \ref{lem: Smooth f } to get 
$$|l_t(z_t) - l_t(x_t)| \ \leq \frac{L}{2}\tau^2 ||\mathbf{e}_t||^2_2 = \frac{L}{2}\tau^2 .$$

Next we sum up all steps from $1$ to $T$. $B_{\psi_p}(u, x_t)$ terminates in telescoping sum except $B_{\psi_p}(u, x_1) - B_{\psi_p}(u, x_{N})$. Since $B_{\psi_p}(u, x_{N}) \geq 0$ we can terminate it too.  For only remaining term $B_{\psi_p}(u, x_1)$ we denote $R_{1,u} = \sqrt{2 B_{\psi_p}(u, x_1)}$ for brevity.
$$\E \left[\frac1T \sum \limits_{t=1}^T (l_t(z_t) - l_t(u))\right]  \leq 4M\tau +\frac{R_{1,u}^2}{2T\mu} + \frac{\mu}{2}  \sigma_{q}^{\alpha+1} \lambda^{1-\alpha} $$
$$+ \left(\frac{\sigma_{q}^{\alpha}}{\lambda^{(\alpha-1)}} + \Delta \frac{\sqrt{n}}{\tau}\right) \mathcal{D}_\psi. $$
In order to get the minimal upper bound we find the optimal value of $\lambda$
$$  \min_{\lambda>0}\sigma_{q}^{\alpha+1} \left( \frac{1}{\lambda^{\alpha}}\mathcal{D}_\psi + \frac{\mu}{2}   \lambda^{1-\alpha} \right) = \min_\lambda \sigma_{q}^{\alpha+1} h_1(\lambda)$$

$h_1'(\lambda) = \frac{\mu}{2} (1-\alpha) \lambda^{-{\alpha}} - {\alpha}\frac{1}{\lambda^{\alpha+1}}\mathcal{D}_\psi = 0 \Rightarrow \lambda^* =  \frac{2{\alpha}\mathcal{D}_\psi}{(1-\alpha)\mu}.$

$$\E \left[\frac1T \sum \limits_{t=1}^T (l_t(z_t) - l_t(u))\right]  \leq 4M\tau + \frac12 \frac{R_{1,u}^{2} }{\mu T} + \Delta \frac{\sqrt{n}}{\tau} \mathcal{D}_\psi $$
$$+ \sigma_{q}^{\alpha+1} \left( \mathcal{D}^{1-\alpha} 2^{-{\alpha}} \mu^{{\alpha}} \left[ \frac{(1-\alpha)^{\alpha}}{{\alpha}^{\alpha}} + \frac{{\alpha}^{(1-\alpha)}}{(1-\alpha)^{(1-\alpha)}}\right]\right).$$
Considering bound of ${\alpha} \in [0,1]$
$$ \frac{(1-\alpha)^{\alpha}}{{\alpha}^{\alpha}} + \frac{{\alpha}^{(1-\alpha)}}{(1-\alpha)^{(1-\alpha)}} \leq 2. $$

\begin{equation}\label{clip before nu}
    \E \left[\frac1T \sum \limits_{t=1}^T (l_t(z_t) - l_t(u))\right]  \leq 4M\tau + \frac12 \frac{R_{1,u}^{2} }{\mu T} + \Delta \frac{\sqrt{n}}{\tau} \mathcal{D}_\psi + \sigma_{q}^{\alpha+1} \left( 2\mathcal{D}_\psi^{1-\alpha}  \mu^{{\alpha}} \right).\end{equation}
Choosing the optimal $\mu^*$ similarly we get
$$\mu^* = \left( \frac{R_{1,u}^2}{4T{\alpha}\sigma_{q}^{\alpha+1} \mathcal{D}_\psi^{1-\alpha} }\right)^{\frac{1}{\alpha+1}}$$
and
\begin{eqnarray}
\E \left[\frac1T \sum \limits_{t=1}^T (l_t(z_t) - l_t(u))\right]  &\leq& 4M\tau + \Delta \frac{\sqrt{n}}{\tau} \mathcal{D}_\psi \notag\\
&+& 2\frac{R_{1,u}^\frac{2{\alpha}}{\alpha+1} \mathcal{D}_\psi^{\frac{1-\alpha}{\alpha+1}}\sigma_{q} }{T^\frac{{\alpha}}{\alpha+1}} \left[{\alpha}^\frac{1}{\alpha+1} + {\alpha}^{-\frac{{\alpha}}{\alpha+1}} \right] .\notag
\end{eqnarray}
Considering bound of ${\alpha} \in [0,1]$
$${\alpha}^\frac{1}{\alpha+1} + {\alpha}^{-\frac{{\alpha}}{\alpha+1}} \leq 2$$
and as a consequence
\begin{equation}\label{after nu}
    \E \left[\frac1T \sum \limits_{t=1}^T (l_t(z_t) - l_t(u))\right]  \leq 4M\tau + \Delta \frac{\sqrt{n}}{\tau} \mathcal{D}_\psi+2 \frac{R_{1,u}^\frac{2{\alpha}}{\alpha+1} \mathcal{D}_\psi^{\frac{1-\alpha}{\alpha+1}}\sigma_{q} }{T^\frac{{\alpha}}{\alpha+1}}. \end{equation}
In order to avoid $\mu \rightarrow \infty$ when ${\alpha} \rightarrow 0$ one can also choose $\mu^* = \left( \frac{R_{1,u}^2}{4T\sigma_{q}^{\alpha+1} \mathcal{D}_\psi^{1-\alpha} }\right)^{\frac{1}{\alpha+1}}$. Estimation \eqref{after nu} doesn't change.
Finally, we bound $\sigma_q$ with Lemma \ref{lem: grad 1 + k norm}
$$\sigma_q \leq 2\frac{na_{q}B}{\tau} + 2\frac{na_{q}\Delta}{\tau}= \frac{2 n a_q}{\tau} (\Delta + B),$$
Thus
\begin{equation}\label{after nu 2}
    \E \left[\frac1T \sum \limits_{t=1}^T (l_t(z_t) - l_t(u))\right]  \leq 4M\tau + \Delta \frac{\sqrt{n}}{\tau} \mathcal{D}_\psi+4 \frac{R_{1,u}^\frac{2{\alpha}}{\alpha+1} \mathcal{D}_\psi^{\frac{1-\alpha}{\alpha+1}}n a_q (\Delta + B) }{\tau T^\frac{{\alpha}}{\alpha+1}} \end{equation}
and set optimal $\tau_M$
$$\tau_M = \sqrt{\frac{\sqrt{n}\Delta\mathcal{D}_{\psi} + 4R_{1,u}^\frac{2{\alpha}}{\alpha+1} \mathcal{D}_\psi^{\frac{1-\alpha}{\alpha+1}}na_{q}(\Delta + B) T^{-\frac{{\alpha}}{\alpha+1}}}{4M}}.$$
Also in Smooth case the bound \eqref{after nu 2} changes into

\begin{equation}\label{after nu 2 smooth case}
    \E \left[\frac1T \sum \limits_{t=1}^T (l_t(z_t) - l_t(u))\right]  \leq 2 L\tau^2 + \Delta \frac{\sqrt{n}}{\tau} \mathcal{D}_\psi+4 \frac{R_{1,u}^\frac{2{\alpha}}{\alpha+1} \mathcal{D}_\psi^{\frac{1-\alpha}{\alpha+1}}n a_q (\Delta + B) }{\tau T^\frac{{\alpha}}{\alpha+1}}, \end{equation}
and  optimal $\tau_L$ becomes
$$\tau_L = \sqrt[3]{\frac{\sqrt{n}\Delta\mathcal{D}_{\psi} + 4R_{1,u}^\frac{2{\alpha}}{\alpha+1} \mathcal{D}_\psi^{\frac{1-\alpha}{\alpha+1}}na_{q}(\Delta + B) T^{-\frac{{\alpha}}{\alpha+1}}}{2L}}.$$

\end{proof}
%\section{Conclusion}

\section{Acknowledgements}

The work of A. Gasnikov was supported by a grant for research centers in the field of artificial intelligence, provided by the Analytical Center for the Government of the Russian Federation in accordance with the subsidy agreement (agreement identifier 000000D730321P5Q0002 ) and the agreement with the Ivannikov Institute for System Programming of the Russian Academy of Sciences dated November 2, 2021 No. 70-2021-00142.

\bibliography{sn-article}% common bib file

\end{document}